\documentclass{article} % For LaTeX2e
\usepackage{iclr2027_conference,times}

\usepackage{amsmath,amsfonts,bm}

\def\eqref#1{equation~\ref{#1}}
\def\1{\bm{1}}

\DeclareMathAlphabet{\mathsfit}{\encodingdefault}{\sfdefault}{m}{sl}
\SetMathAlphabet{\mathsfit}{bold}{\encodingdefault}{\sfdefault}{bx}{n}

\usepackage{hyperref}
\usepackage{url}
\usepackage{amsthm}
\usepackage{amssymb}
\usepackage[ruled,linesnumbered]{algorithm2e}
\usepackage{booktabs}
\usepackage{colortbl}
\usepackage{multirow}
\usepackage{graphicx}
\usepackage{wrapfig}
\usepackage{listings}
\usepackage{tcolorbox}
\tcbuselibrary{skins, breakable}
\usepackage{cuted}
\usepackage{caption}
\usepackage{array}
\usepackage{svg}

\title{Self-Evolving Visual Agent via Margin Reward}

\iclrfinalcopy

\author{Wenxi Yang \\
UCAS \\
\texttt{yangwenxi24@mails.ucas.ac.cn}
\And
Yuzhong Zhao \\
Microsoft Research Asia \\
\texttt{yuzhongzhao@microsoft.com}
\And
Baochang Zhang \\
Beihang University \\
\texttt{bczhang@buaa.edu.cn}
\And
Yunfan Liu \\
UCAS \\
\texttt{liuyunfan@ucas.ac.cn}
\And
Qixiang Ye \\
UCAS \\
\texttt{qxye@ucas.ac.cn}
}

\renewcommand{\eqref}[1]{Eq.~(\ref{#1})}

\newtheorem{proposition}{Proposition}

\begin{document}

\def\methodname{Marvo{}}

\maketitle

\begin{abstract}
Visual agents that interleave reasoning with tool invocation have posed a promising direction for multimodal tasks. 
However, learning effective tool invocation usually relies on curated question–answer pairs or expert demonstrations at scale, which are expensive to acquire. %annotate and limited in scalability. 
In this study, we propose \methodname{}, a simple-yet-effective self-evolution approach that enables learning tool invocation from unlabeled images. 
%images
\methodname{} consists of two agents: (1) a challenger that generates question-answer pairs from unlabeled images, and (2) a solver that learns tool invocation by solving questions through reinforcement learning.
%
% To encourage the solver to learn effective tool invocation, \methodname{} favors training questions those are difficult enough to warrant tool use, yet remain within the solver's capabilities.
To encourage effective tool invocation, a good training question should be difficult enough to warrant tool use, yet remain within the solver's capabilities.
% when equipped with tools.
% To promote effective tool invocation, training questions should be difficult enough to warrant tool use, yet remain within the solver's capabilities when equipped with tools.
%
To this end, we introduce a margin reward to the challenger, defined as solver's performance difference with and without access to tools.
With the margin reward, we create an adaptive curriculum that encourages the solver to learn effective tool invocation.
Across benchmarks covering agentic reasoning, visual search, mathematical, and general real-world reasoning, \methodname{} improves average accuracy by up to 7.1\% — demonstrating margin reward as a preferred way to scale question complexity without human supervision.
\begin{center}
Code would be released at \url{https://github.com/YWenxi/marvo}.
\end{center}
\end{abstract}

\section{Introduction}
Visual agents have extended the capabilities of multimodal models by interleaving reasoning with tool invocation \citep{o3, zhao2025pyvision}. Instead of producing answers solely from an initial visual observations, these agents have already invoked tools to acquire rich evidence, process visual information, and use the feedback to guide subsequent reasoning. 
%took a further step to 
This interaction offers a promising approach to multimodal tasks, but requires agents to learn effective tool invocation--they have to determine which tool is useful and how to incorporate its output into the reasoning process.

Existing approaches learning these behaviors usually rely on curated question--answer pairs through human supervision or expert demonstrations~\citep{zheng2026deepeyes, hong2026deepeyesv2, zhao2026pyvisionrl}. 
Such supervision provides task objectives or examples of successful tool use, but is expensive to be collected at scale.
%collecting it at scale is expensive. 
%
Expert demonstrations require constructing suitable interaction trajectories, while question--answer supervision requires identifying meaningful questions and establishing answers. 
These challenges motivate learning with zero data~\citep{huang2026r} or unlabeled data. Yet converting this potential into useful supervision requires a mechanism for producing valuable tasks and questions, particularly, for visual agents that largely depend on tool invocation.
% that help improving generalization capability of agents
%These challenges motivate learning from a more readily available resource: unlabeled images. An image can support multiple reasoning questions, making it a potential source of training tasks without additional human annotation. Yet converting this potential into useful supervision requires a mechanism for selecting questions that help an agent improve its tool invocation.

The value of a training question depends not only on its difficulty, but also on whether tools help the solver to correctly answer it. On the one hand, questions that the solver can answer without tools provide little incentive to acquire tool-use behaviors. On the other hand, questions that remain beyond the solver's capabilities even with tools offer limited opportunities to learn from successful interactions. 
Useful questions should therefore occupy a region between these extremes: they should pose a challenge beyond the solver's unaided reasoning while remaining solvable with tool invocation.
Moreover, this region depends on the solver's current capabilities and may shift as the solver improves. A productive training curriculum should track this changing relationship between question difficulty and tool utility.

In this study, we propose \methodname{}, a simple yet effective self-evolution approach that enables visual agents learning tool invocation from unlabeled images (Fig.~\ref{fig:figure1}a). 
\methodname{} consists of two agents with complementary roles. A challenger generates question--answer pairs from unlabeled images, and a solver learns to answer the generated questions through reinforcement learning with access to tools (Fig.~\ref{fig:figure1}b). The challenger supplies training tasks, while the solver's performance provides feedback about their learning values. This interaction allows the training curriculum to adapt to the solver, reducing dependence on a fixed collection of manually curated tasks.

To guide the challenger toward questions that promote tool invocation, we introduce a \emph{margin reward}, which measures solver's performance difference with and without access to tools. A larger margin, $i.e.$, performance difference, indicates that tool access enables the solver to answer a question more successfully than it could through unaided reasoning. Questions that are easily answered in both settings, or remain unsuccessful in both, receive little reward from this objective. Trained with the margin reward, the challenger is encouraged to generate questions that expose opportunities for useful tool invocation.

This reward connects question generation directly to the behavior we seek to improve. Optimizing question difficulty alone can favor tasks that the solver cannot solve, while optimizing tool-assisted success alone can favor tasks that do not require tools. The margin reward instead measures the solver's relative benefit from tool access. As this benefit is evaluated with respect to the current solver, the same objective can guide question generation as the solver's capabilities evolve. The resulting curriculum emphasizes questions for which tools make a practical difference, 
enabling visual agents to learn effective tool invocation without manually annotated question-answer pairs or expert demonstrations.

Across a diverse suite of benchmarks--11,691 questions spanning agentic reasoning~\citep{li2025tir}, visual search~\cite{hrbench, DBLP:conf/cvpr/WuX24a}, mathematical reasoning~\cite{DBLP:conf/eccv/ZhangJZLGQZLCQGL24, qiao2024we, DBLP:conf/nips/WangPSLRZZL24}, and general real-world reasoning--\methodname{} improves the average accuracy of Qwen2.5-VL-7B by 7.1\%. Together, these results establish the performance gap between tool-free and tool-augmented answering as a useful signal for generating training questions that improve tool-augmented reasoning.

% Our contributions are threefold:
% \begin{itemize}
%     \item We propose \methodname{}, a self-evolution approach in which a challenger and a solver turn unlabeled images into training opportunities for learning tool invocation.
%     \item We introduce a margin reward that guides question generation using the solver's performance difference with and without tools, encouraging an adaptive curriculum centered on tool utility.
%     \item We demonstrate accuracy improvements of up to 5.7\% on TIR-Bench and provide further evidence of the approach's applicability through experiments on mathematical reasoning.
% \end{itemize}

\begin{figure}[t]
\centering
\includegraphics[width=\linewidth]{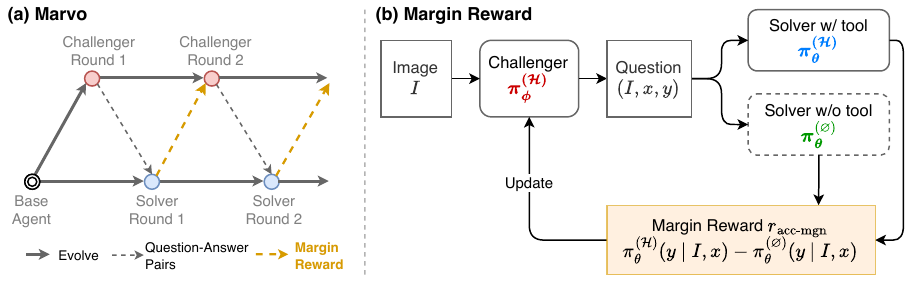}
\caption{\textbf{Flowchart of \methodname{}.}
(a) \methodname{} employs an evolution loop between a challenger and a solver, driven by the Margin Reward.
(b) The Margin Reward quantifies the solver's performance gap with and without tools, encouraging the challenger to propose questions that require tool use.
}
\label{fig:figure1}
\end{figure}

\section{Methodology}

\subsection{Preliminary}

% \begin{figure}[h]
% \centering
% \includegraphics[width=\linewidth]{figs/method4.pdf}
% \caption{
% (a) $r_{\text{uncertainty}}$ seeks for questions of medium difficulty for $\pi_\theta^{(\mathcal{H})}$, while $r_{\text{acc-mgn}}$ favors questions where tool access helps most.
% (b) The challenger proposes a question–answer pair $(x, y)$ from an unlabeled image $I$. The margin reward is used to update the challenger.
% }
% \label{fig:method}
% \end{figure}

\methodname{} is built upon the self-evolution framework of R-Zero~\citep{huang2026r}, which comprises two roles: a challenger $\pi_\phi$ and a solver $\pi_\theta$. Given an input image $I$, the challenger proposes question--answer pairs 
$\{(x, y)\}$, which serve as the training curriculum for the solver. 
To drive self-evolution, the challenger and the solver are optimized alternately with the following objectives, as
\begin{equation}
    \label{eq:objection-general-form}
    \begin{aligned}
        &\max_{\phi} \mathbb E_{I\sim\mathcal I, (x, y) \sim\pi_{\phi}(\cdot \mid I)} \left[ r\left(x,y, \pi_{\theta}\right)\right], \\
        &\max_{\theta} \mathbb E_{I\sim\mathcal I, (x, y) \sim\pi_{\phi}(\cdot \mid I),\hat y \sim\pi_\theta(\cdot \mid I, x)} \left[\1_{\{y =\hat y\}}\right],
    \end{aligned}
\end{equation}
where $r$ is the challenger's reward and $\mathcal I$ denotes the image pool. 

\begin{wrapfigure}{r}{0.4\textwidth}
\centering
\vspace{-0.5cm}
\includegraphics[width=0.36\textwidth]{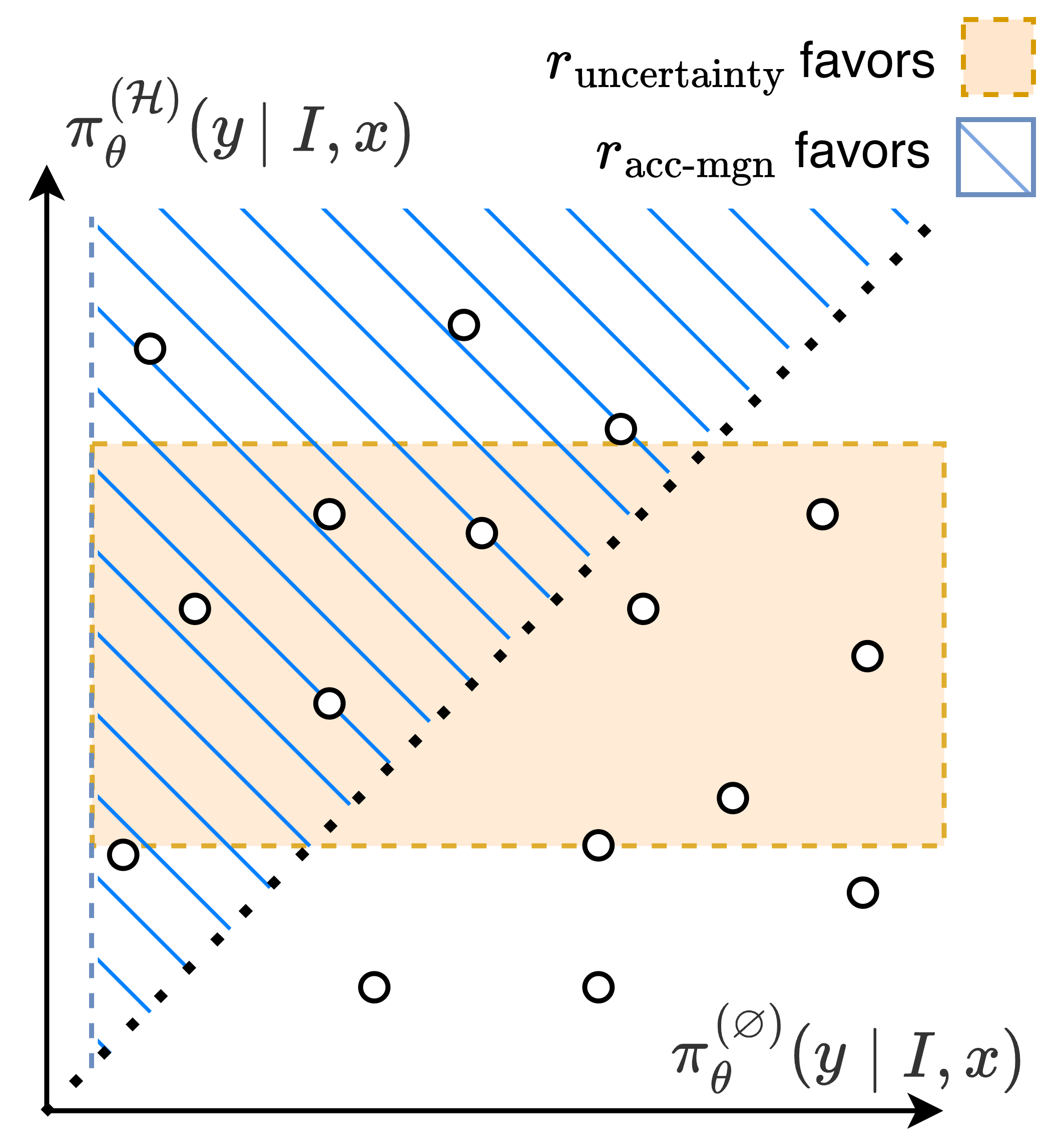}
\caption{\bf Question Distribution.}
\label{fig:method}
\vspace{-1.2cm}
\end{wrapfigure}

The R-Zero method designs $r$ as the uncertainty score, as
\begin{equation}
    r_{\text{uncertainty}}\left(x,y, \pi_{\theta}\right) = 1 - 2 \left| \max_{\tilde y} \big(\pi_\theta(\tilde y \mid I, x)\big) - \frac 12 \right|.
\end{equation}
Note that $r_{\text{uncertainty}}$ uses a pseudo label $\tilde y$ voted by the solver, instead of the challenger proposed label $y$. 
$r_{\text{uncertainty}}$ attains its maximum when $\max_{\tilde y}\pi_\theta(\tilde y \mid I, x) = 1/2$.
In other word, $r_{\text{uncertainty}}$ favors problems of intermediate difficulty for $\pi_\theta$, Fig.~\ref{fig:method}(a).
However, uncertainty alone is an indirect and often unreliable proxy for learning tool invocation. It captures only how confused the solver is, not whether solving a problem actually warrants tool use or would improve the solver's ability to invoke tools.

\subsection{Incentivizing with Margin Reward}
\label{sec:method-margin}

Building on \eqref{eq:objection-general-form}, we investigate how the challenger's reward should be designed so that it incentivizes effective tool invocation.

In specific, we consider instantiating a solver $\pi_\theta$ deployed with different harnesses, $\pi_\theta^{(\mathcal H)}$ and $\pi_\theta^{(\varnothing)}$, where $\mathcal H$ denotes the harness that grants tool access and $\varnothing$ denotes the empty harness, \textit{i.e.}, answering directly without any tools.
The goal is to train the target solver to master the tool-invocation capabilities that $\mathcal H$ affords beyond $\varnothing$. We therefore encourage the challenger $\pi_\phi$ to generate a problem $(x, y)$ such that it is solvable under $\mathcal H$ and not solvable under $\varnothing$. Namely,
\begin{equation}
\label{eq:support-gap}
    \textit{find $(x,y)$ such that } \quad 
    y \in \operatorname{supp}\!\Big(\pi_\theta^{(\mathcal H)}(\cdot \mid I, x)\Big) 
    \;\big\backslash\; 
    \operatorname{supp}\!\Big(\pi_\theta^{(\varnothing)}(\cdot \mid I, x)\Big),
\end{equation}
where $\operatorname{supp}(\cdot)$ denotes the support of a distribution.

As the strict support condition in \eqref{eq:support-gap} is rarely satisfied exactly, we relax it using the expected accuracy $\pi_{\theta}^{(\mathcal H)}(y \mid I, x)$. For a problem $(x, y)$, we design the challenger's reward as the margin in the solver's accuracy with and without access to tools, as
\begin{equation}
    \label{eq:accuracy-margin}
    r_{\text{acc-mgn}}(y, \pi_\theta) \;=\; 
    \pi_{\theta}^{(\mathcal H)}(y \mid I, x) - \pi_{\theta}^{(\varnothing)}(y \mid I, x).
\end{equation}
%precisely
The margin in \eqref{eq:accuracy-margin} is maximized accordingly when the target solver handles the problem reliably with tools while it fails without them (Fig.~\ref{fig:method}(a)), thereby steering the challenger toward questions that are difficult enough to warrant tool use, yet within the solver's capabilities when equipped with tools.

\paragraph{Connection with {\rm $r_{\text{uncertainty}}$}.}
If we rewrite the policy probability in \eqref{eq:accuracy-margin} using log-odds, \textit{i.e.} $\operatorname{logit}(p) = \log (p / (1-p))$, and perform a first-order Taylor expansion, then \eqref{eq:accuracy-margin} is interpreted as
\begin{equation}
\label{eq:margin-logodds}
r_{\text{acc-mgn}}(y, \pi_\theta) \;\approx\; \underbrace{\pi_{\theta}^{(\mathcal H)} \left( 1 - \pi_{\theta}^{(\mathcal H)} \right)}_{\text {uncertainty gate }} 
 \left( \operatorname{logit}\left(\pi_{\theta}^{(\mathcal H)}\right) -\operatorname{logit}\left(\pi_{\theta}^{(\varnothing)}\right) \right).
\end{equation}
Uncertainty is automatically gated by the Bernoulli variance $\pi_{\theta}^{(\mathcal H)}(y|I,x) \big( 1 - \pi_{\theta}^{(\mathcal H)}(y|I,x)\big)$ in \eqref{eq:margin-logodds}, which guarantees the proposed questions are effective incentives to train the solver.

\subsection{Equilibrium Analysis}
\label{sec:equilibrium}

We briefly characterize the equilibrium state induced by \eqref{eq:objection-general-form} and \eqref{eq:accuracy-margin}, deferring the full analysis to Appendix~\ref{app:nash}. 
At equilibrium, each proposed problem $(x,y)$ falls to one of two categories, as visualized in Fig.~\ref{fig:equilibrium}. 

In a \emph{gap-closed} equilibrium, $$\pi_{\theta}^{(\mathcal H)}(y|I,x) - \pi_{\theta}^{(\varnothing)}(y|I,x) \equiv 0.$$ 
\begin{wrapfigure}{r}{0.4\textwidth}
\centering
\includegraphics[width=0.35\textwidth]{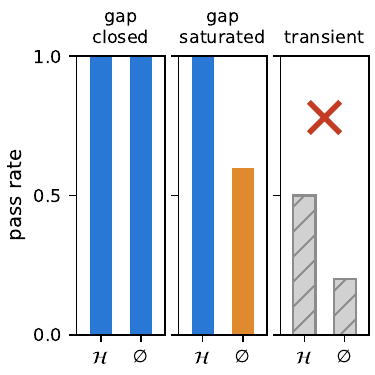}
\caption{\bf Equilibrium States.}
\label{fig:equilibrium}
\vspace{-0.3cm}
\end{wrapfigure}
The two harnesses are indistinguishable.
 In the solved case, where $\pi_{\theta}^{(\mathcal H)}(y|I,x) = \pi_{\theta}^{(\varnothing)}(y|I,x) = 1$, 
the tool-use capability that $\mathcal H$ affords has been learned by the model itself: under $\varnothing$, it can now answer directly without any external aid.

In a \emph{gap-saturated}
equilibrium, 
$$\pi_{\theta}^{(\mathcal H)}(y|I,x) = 1 > \pi_{\theta}^{(\varnothing)}(y|I,x).$$ 
The solver has mastered the proposed problems 
with tool access
while 
% $\varnothing$ fails 
failing them without tool access,
% despite sharing the same weights. 
which exposes an irreducible gap attributable solely to tool access.
Both players are simultaneously at their optima, but the solver's gradient vanishes, so learning stalls.
% until the problem frontier is expanded. 
%
Notably, equilibria never contain partially solved problems, where $0 < \pi_{\theta}^{(\mathcal H)}(y|I,x) < 1$. 
Such states are transient, and the uncertainty gate in \eqref{eq:margin-logodds} is active precisely along these transients, where the
solver still perceives learning signal.

\subsection{Training Paradigm}

Building upon $r_{\text{acc-mgn}}$, we introduce {\bf \methodname} (Self-E\underline{\bf vo}lution via \underline{\bf Mar}gin Reward), a training paradigm where a challenger and a solver co-evolve, as summarized in Alg.~\ref{alg:accuracy-margin}.
Both agents are instantiated from the same base model, yet play opposing roles: the challenger $\pi_\phi^{(\mathcal H)}$ learns to propose difficult questions, while the solver $\pi_\theta^{(\mathcal H)}$ learns to answer them. Notably, the challenger is also equipped with the harness $\mathcal H$, since posing questions that genuinely warrant tool use requires the challenger to experience the tools firsthand. The two agents are then trained \emph{alternately} through group-mean policy optimization (GRPO)\footnote{See details of GRPO implementation in Appendix~\ref{app:preliminaries}.}. 

\paragraph{Challenger Phase.}
Given an unlabeled image $I$, the challenger freely investigates the image with $\mathcal H$ and proposes question-answer pairs in multiple-choice format. Each proposed question is answered $n$ times each by $\pi_\theta^{(\mathcal H)}$ and $\pi_\theta^{(\varnothing)}$.
The resulting $r_{\text{acc-mgn}}$, combined with a format punishment $r_{\text{format}}$, is computed as the terminal reward for updating the challenger. Specifically, $r_{\text{format}} = 0$ if a valid question is extracted, and $-1$ otherwise.

\paragraph{Dataset Curation.}
We apply the challenger to a pool of images to propose candidate questions. Each question is answered $n$ times by the target solver $\pi_\theta^{(\mathcal H)}$, and questions with a pass rate of $0$ or $1$ are discarded, as such questions are either too hard or too easy for the current solver and their uniform solver rewards produce zero advantage in GRPO.

\paragraph{Solver Phase.}
The target solver $\pi_{\theta}^{(\mathcal H)}$ is  then updated with GRPO on the curated questions,  using the rule-based reward $\1_{\{\hat y = y\}}$.
We only train $\pi_\theta^{\left(\mathcal H\right)}$, as mixing trajectories from $\mathcal H$ and $\varnothing$ in the same batch may lead to conflicting gradient updates.

\begin{algorithm}[ht]
\caption{\methodname{}: Self-Evolution via Margin Reward}
\label{alg:accuracy-margin}
\KwIn{image pool $\mathcal{I}$; base model $\mathcal{M}$; rounds $R$; images per round $B$; challenger steps $m$, solver steps $n$; group size $G$; target harness $\mathcal H$}
\KwOut{trained solver $\pi_{\theta}^{(\mathcal{H})}$, trained challenger $\pi_{\phi}^{(\mathcal{H})}$}
Initialize $\pi_{\phi}$ and $\pi_{\theta}$ from the same base model $\mathcal M$\;
\For{$r \leftarrow 1$ \KwTo $R$}{
  $\mathcal{B}_r \leftarrow$ next $B$ images of the shuffled pool $\mathcal{I}$\;
  \tcp{Challenger Phase}
  \For{$m$ GRPO steps}{
    Sample images $I \sim \mathcal{B}_r$\;
    Sample $G$ question--answer pairs $\{(x_i, y_i)\}_{i=1}^G$ per image with $\pi_{\phi}^{(\mathcal{H})}$\;
    For each proposed question, compute $\pi_{\theta}^{(\mathcal{H})}(y_i \mid I, x_i)$ and $\pi_{\theta}^{(\varnothing)}(y_i \mid I, x_i)$ \;
    $r_i \leftarrow r_{\text{acc-mgn}}(y_i, \pi_{\theta}) + r_{\text{format}}$\;
    Update $\phi$ by GRPO\;
  }
  
  \tcp{Dataset Curation}
  $\mathcal{D}_r \leftarrow$ problems $(x, y)$ authored by $\pi_{\phi}^{(\mathcal{H})}$ on $\mathcal{B}_r$\;
  $\mathcal{D}_r \leftarrow$ filter out malformed and trivial questions in $\mathcal{D}_r$\;
  
  \tcp{Solver Phase}
  \For{$n$ GRPO steps}{
    Sample $(x, y) \sim \mathcal{D}_r$ (one pass; no epochs)\;
    Sample $G$ episodes per problem $\hat{y}_i \sim \pi_\theta^{(\mathcal{H})}(\cdot \mid I, x)$\;
    $r_i \leftarrow \1_{\{y = \hat{y}_i\}}$\;
    Update $\phi$ by GRPO\;
  }
  
}
\end{algorithm}

\section{Experiment}
\label{sec:experiment}

\subsection{Setting}

\paragraph{Models \& Datasets.}
PyVision~\citep{zhao2025pyvision} is adopted as the target harness $\mathcal H$ to incentivize, which is an interactive, multi-turn framework that enables VLMs to invoke Python-based tools. 
Qwen2.5-VL-7B-Instruct~\citep{Qwen2.5-VL} and Qwen3-VL-4B~\citep{Qwen3-VL} are adopted as base models. Training are initialized from the base models' Supervised Fine-Tuning~(SFT) checkpoints\footnote{See details of SFT checkpoints in Appendix~\ref{app:sft-training-details}.}.
We sample images from the PyVision-RL training set~\citep{zhao2026pyvisionrl} and discard those with excessively large resolutions, yielding a pool of 42{,}124 images. The pool spans high-resolution natural images, mathematical diagrams, and scientific charts. 

\paragraph{Evaluation.}
To demonstrate \methodname{}'s efficacy, the trained models are evaluated on a diverse set of benchmarks. 
\textit{Agentic reasoning} is measured on TIR-Bench~\citep{li2025tir}, which requires actively manipulating visual inputs with tools rather than relying on static observation. 
\textit{Visual search} is assessed on V*~\citep{DBLP:conf/cvpr/WuX24a} and HRBench-4K~\citep{hrbench}, which test the grounding of fine-grained details in high-resolution and information-dense images. 
\textit{Multimodal mathematical reasoning} is evaluated on MathVerse~\citep{DBLP:conf/eccv/ZhangJZLGQZLCQGL24}, MathVision~\citep{DBLP:conf/nips/WangPSLRZZL24}, and WeMath~\citep{qiao2024we}, where diagrams, charts, and plots must be interpreted to solve math problems. 
\textit{real-world understanding} of everyday scenes is tested on RealWorldQA~\citep{RealWorldQA2024}.

\subsection{Main Results}
\label{sec:main-results}

\newcommand{\up}[1]{\textcolor{green!60!black}{\scriptsize(+#1)}}
\newcommand{\down}[1]{\textcolor{red!80!black}{\scriptsize(--#1)}}
\newcommand{\tied}[1]{\textcolor{gray}{\scriptsize(#1)}}
\newcommand{\best}[1]{\textbf{#1}}
\begin{table*}[t]
\centering\small
\setlength{\tabcolsep}{3pt}
\caption{\textbf{Main results.} For each self-evolution process (one model family with one method), the best value in each column is in \textbf{bold}. 
Parenthesized values are the change versus the previous round. 
% \textcolor{green!60!black}{green} for gains and \textcolor{red!80!black}{red} for drops. 
Overall is the mean over the seven benchmarks.
}
\label{tab:main-results}
\begin{tabular}{@{}>{\footnotesize}llllllll|l@{}}
\toprule
& \multicolumn{1}{c}{Agentic} & \multicolumn{2}{c}{Visual Search} & \multicolumn{3}{c}{Mathematical Reasoning} & \multicolumn{1}{c}{General} & \\
\cmidrule(lr){2-2}\cmidrule(lr){3-4}\cmidrule(lr){5-7}\cmidrule(lr){8-8}
& TIR-Bench & V* & HR4K & MathVerse & MathVision & WeMath & RWQA & Overall \\
\midrule
Gemini-2.5-Pro & 28.9 & 83.8 & 87.3 & 82.9 & 63.8 & 78.0 & 78.0 & 71.6 \\
o3 & 26.9 & 95.7 & - & - & 73.3 & 80.6 & - & - \\
Qwen2.5-7B & 16.0 & 78.5 & 71.6 & 33.8 & 25.6 & 34.6 & 57.5 & 34.8 \\
Qwen3-4B & - & 74.9 & 73.5 & 42.3 & 44.6 & - & 70.9 & - \\
DeepEyesV2 & - & 81.8 & 77.9 & 52.7 & 28.9 & 38.1 & - & - \\
Agent0-VL & - & - & - & 53.1 & 37.3 & 71.7 & - & - \\
PyVision-RL & 19.8 & 88.7 & 78.1 & 55.8 & 28.7 & 47.7 & 66.3 & 46.6 \\
\midrule
\multicolumn{9}{c}{\emph{Self-Evolution with Qwen2.5-VL-7B}} \\
\midrule
\multicolumn{9}{l}{\emph{R-Zero}} \\
SFT & 17.1 & 72.7 & 67.0 & 41.5 & 25.1 & 31.9 & 57.5 & 36.6 \\
Round 1 & 17.4\up{0.3} & 73.3\up{0.6} & 67.5\up{0.5} & 43.1\up{1.6} & \best{25.6}\up{0.5} & 47.7\up{15.8} & 55.8\down{1.7} & 39.6\up{3.0} \\
Round 2 & 18.5\up{1.1} & 74.9\up{1.6} & 67.5\tied{+0.0} & 43.1\tied{+0.0} & 23.7\down{1.9} & 51.1\up{3.4} & 57.9\up{2.1} & 39.8\up{0.2} \\
Round 3 & \best{20.7}\up{2.2} & \best{76.4}\up{1.5} & \best{70.1}\up{2.6} & \best{44.2}\up{1.1} & \best{25.6}\up{1.9} & \best{56.5}\up{5.4} & \best{61.4}\up{3.5} & \best{42.2}\up{2.4} \\
\midrule
\multicolumn{9}{l}{\emph{\methodname{} (Ours)}} \\
SFT & 17.1 & 72.7 & 67.0 & 41.5 & 25.1 & 31.9 & 57.5 & 36.6 \\
Round 1 & 20.2\up{3.1} & 73.8\up{1.1} & 67.3\up{0.3} & 43.2\up{1.7} & 25.2\up{0.1} & 48.7\up{16.8} & 58.0\up{0.5} & 40.1\up{3.5} \\
Round 2 & 20.6\up{0.4} & 75.4\up{1.6} & 68.6\up{1.3} & 43.6\up{0.4} & 25.6\up{0.4} & 53.4\up{4.7} & \best{65.1}\up{7.1} & 41.6\up{1.5} \\
Round 3 & \best{21.7}\up{1.1} & \best{78.5}\up{3.1} & \best{73.8}\up{5.2} & \best{44.2}\up{0.6} & \best{27.8}\up{2.2} & \best{58.6}\up{5.2} & 64.6\down{0.5} & \best{43.7}\up{2.1} \\
\midrule
\multicolumn{9}{c}{\emph{Self-Evolution with Qwen3-VL-4B}} \\
\midrule
\multicolumn{9}{l}{\emph{R-Zero}} \\
SFT & \best{22.3} & 78.0 & 71.0 & 49.7 & 38.8 & 57.4 & 62.2 & 47.9 \\
Round 1 & 22.2\down{0.1} & \best{83.3}\up{5.3} & 75.3\up{4.3} & 52.9\up{3.2} & 40.0\up{1.2} & 59.5\up{2.1} & \best{68.1}\up{5.9} & 50.4\up{2.5} \\
Round 2 & 21.6\down{0.6} & 78.5\down{4.8} & \best{77.0}\up{1.7} & \best{53.6}\up{0.7} & 39.7\down{0.3} & 64.3\up{4.8} & 62.6\down{5.5} & 50.8\up{0.4} \\
Round 3 & 22.1\up{0.5} & 79.1\up{0.6} & 76.8\down{0.2} & 52.7\down{0.9} & \best{40.3}\up{0.6} & \best{67.4}\up{3.1} & 66.3\up{3.7} & \best{51.4}\up{0.6} \\
\midrule
\multicolumn{9}{l}{\emph{\methodname{} (Ours)}} \\
SFT & 22.3 & 78.0 & 71.0 & 49.7 & 38.8 & 57.4 & 62.2 & 47.9 \\
Round 1 & 22.9\up{0.6} & 81.7\up{3.7} & 75.5\up{4.5} & 52.9\up{3.2} & 38.1\down{0.7} & 59.0\up{1.6} & \best{65.1}\up{2.9} & 49.7\up{1.8} \\
Round 2 & 22.9\tied{+0.0} & 83.3\up{1.6} & 77.3\up{1.8} & 53.5\up{0.6} & \best{39.7}\up{1.6} & 65.5\up{6.5} & 64.6\down{0.5} & 51.3\up{1.6} \\
Round 3 & \best{23.8}\up{0.9} & \best{83.8}\up{0.5} & \best{77.9}\up{0.6} & \best{53.7}\up{0.2} & 39.6\down{0.1} & \best{68.8}\up{3.3} & 64.2\down{0.4} & \best{52.0}\up{0.7} \\
\bottomrule
\end{tabular}
\end{table*}

In Table~\ref{tab:main-results}, \methodname{} is compared against R-Zero across two model families, alongside evaluation results of frontier models~\citep{comanici2025gemini, o3} and open-source tool-augmented visual agents~\citep{liu2025agent0, hong2026deepeyesv2, zhao2026pyvisionrl} for reference.

\paragraph{\methodname{} prevails at matched iteration.}
Across both model families, \methodname{} consistently outperforms R-Zero at matched iterations, with the largest gains appearing exactly where agency matters most: +3.7 on HR4K and +2.1 on V* for the 7B model at Round 3, and +4.7 on V* for the 4B model. The cumulative picture is even clearer. The advantage is most pronounced on agentic reasoning: \methodname{} delivers gains on both Qwen2.5-VL-7B and Qwen3-VL-4B, whereas R-Zero with Qwen3-VL-4B finishes below its starting point on TIR-Bench. Surprisingly, our evolved 7B model surpasses PyVision-RL on TIR-Bench, despite both being trained from the same SFT checkpoint.

\paragraph{\methodname{} shows better stability across rounds.}
The per-round deltas in Tab.~\ref{tab:main-results} expose a qualitative difference between the two evolution loops. 
R-Zero's improvements are often non-monotic: its 4B run regresses on four of seven benchmarks in Round 2 alone and its 7B run makes no progress on three benchmarks in Round 2. 
Aggregating both model families, 9 of R-Zero's 42 per-round deltas are negative, compared with 5 of 42 under \methodname{}, and the largest single-round regression is 5.5 points under R-Zero versus 0.7 under \methodname{}. 
Notably, R-Zero's regressions concentrate on the tool-dependent benchmarks (V*, RealworldQA, TIR-Bench),
suggesting that its self-generated data distribution shifts substantially between rounds; the near-monotone improvement under \methodname{} indicates a more stable co-evolution curriculum.

\paragraph{Tool invocation is where the gap opens up.}
The sharpest contrast appears on benchmarks that cannot be solved without calling tools. On TIR-Bench, three rounds of R-Zero leave the 4B model slightly \emph{worse} than its SFT start, while \methodname{} improves it by +1.5 and the 7B model by +5.7; TIR-Bench is also where our very first iteration jumps the most. Visual search tells the same story, with cumulative 7B gains of +5.8 on V* and +6.8 on HR4K. 

\subsection{Training Dynamics}

\begin{figure}[t]
\centering
\includegraphics[width=\linewidth]{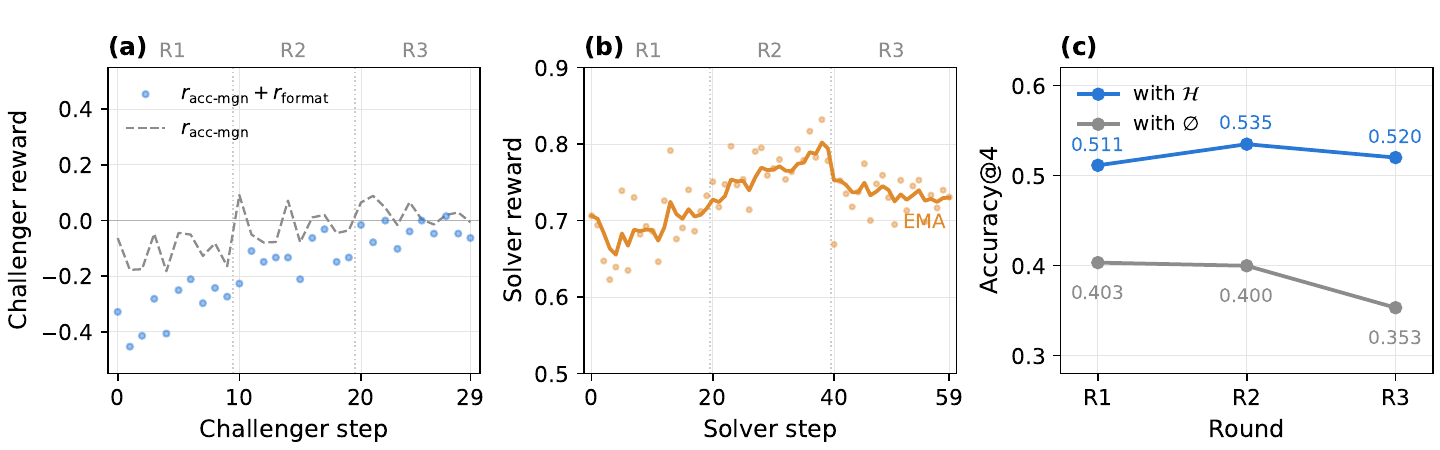}
\caption{
\textbf{Training Dynamics.}
\textbf{(a-b)}~Reward progress in challenger and solver training phase.
\textbf{(c)}~Per-round difficulty of challenger-generated datasets, as measured by model at each round.
}
\label{fig:training-dynamics}
\end{figure}

Figure~\ref{fig:training-dynamics} follows the 4B \methodname{} run through its three evolution rounds, tracking the challenger reward, the solver reward, and a per-round probe of $\pi_{\theta}^{(\mathcal H)}$ versus $\pi_{\theta}^{(\varnothing)}$ accuracy on each round's own questions. The challenger learns on two timescales.
Fig.~\ref{fig:training-dynamics}a shows that in Round~1 the challenger reward climbs steadily while $r_{\text{acc-mgn}}$ stays flat below zero. 
In fact, the fraction of responses from which a question--answer pair could not be extracted 
falls from $0.30$ to $0.17$ over the run, so nearly all early progress is the model learning to emit well-formed, extractable question--answer pairs. Only once the format is in place, does the margin itself move.
Within each round the solver's reward rises as it masters the current question set~(Fig.~\ref{fig:training-dynamics}b).
At the beginning of each round, the reward levels are reset, 
as the updated challenger proposes questions the current solver cannot yet answer well. Across all three rounds, the proposed problems stay within the sweet spot of the uncertainty gate in \eqref{eq:margin-logodds}, where the solver keeps receiving a strong training signal~(Fig.~\ref{fig:training-dynamics}c). When the solver is instead probed without tool access, however, pass rates drop, indicating that the curated problems genuinely require tool invocation.
  
\section{Analysis}

\begin{figure}[t]
\centering
\includegraphics[width=\textwidth]{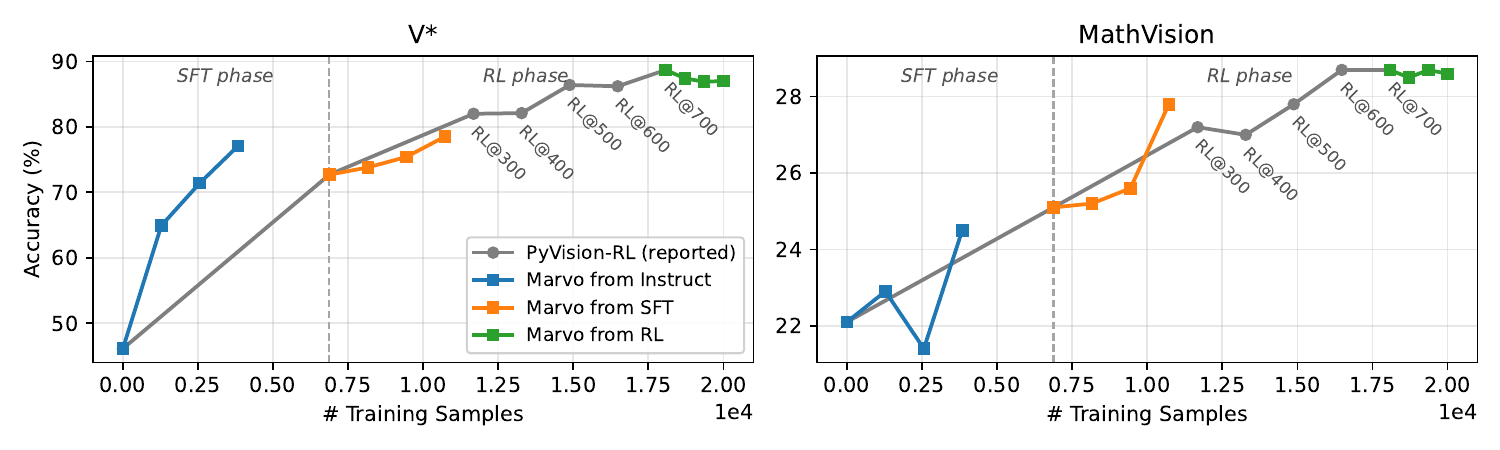}
\caption{
\textbf{\methodname{} across post-training stages.}
The gray curve is the reported PyVision-RL training progress.
% % The dashed line marks the SFT to RL phase boundary. 
Colored curves are \methodname{} experiments launched from different checkpoints. 
}
\label{fig:synergy-curve}
\end{figure}

\begin{figure}[t]
\centering
\includegraphics[width=\linewidth]{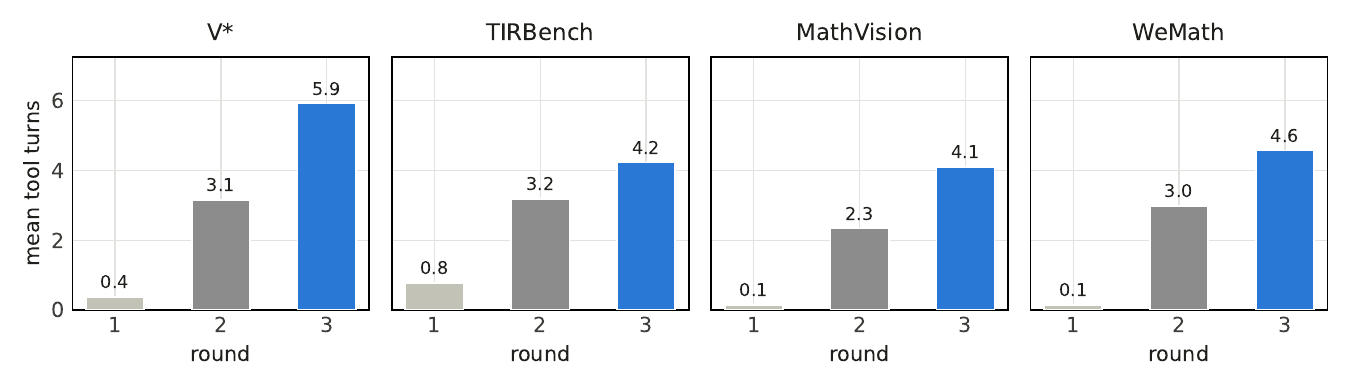}
\caption{
\textbf{Mean number of tool turns during evaluation.} 
Initializing from Instruct model, 
\methodname{} incentivizes tool invocation without human-curated data supervision. 
}
\label{fig:mean-turn}
\end{figure}

%resulting
\paragraph{Self-Evolution Across Post-Training Stages.} 
To understand \emph{when} \methodname{} helps in post-training, we plug it at three stages of the post-training pipeline in PyVision-RL~\citep{zhao2026pyvisionrl} (the Instruct model, the SFT cold-start checkpoint, and the final RL model) and align the evolution curves to the reported training progress
(Fig.~\ref{fig:synergy-curve}). 

At early stages, \methodname{}-generated curriculum incentivizes the visual agent to learn tool invocation much as human-annotated data does, even though it uses no human-annotated question sets or demonstrations at all. 
Note that the average number of tool turns per evaluation question rises across different tasks~(Fig.~\ref{fig:mean-turn}), so the agent's tool invocation ability
does not degrade under self-evolution.  

However, improvement is little when \methodname{} is introduced to the post-trained model. 
We conjecture that this is an entropy effect. 
PyVision-RL reports that policy entropy falls from 0.75 to about 0.4 over the course of RL~\citep{zhao2026pyvisionrl}.
Entropy is a common proxy for a policy's willingness to explore. 
Once entropy has collapsed, there is little exploration headroom left for self-generated data to amplify, and \methodname{}'s influence on the agent fades.
\paragraph{\methodname{} with Another Harness.}
% \paragraph{Tool Configuration Robustness.}
We further evaluate \methodname{} and R-Zero under an alternative tool configuration. 
The VLM is additionally equipped with an image-zoom-in tool, as shown in Appendix.~\ref{app:visual-agents}, which takes a bounding box coordinate as input and returns the cropped image region. Since this tool requires no additional supervision, we train the visual agent directly from the Qwen3-VL-4B checkpoint, bypassing the cold-start stage. As shown in Tab.~\ref{tab:harness-zoomin}, \methodname{} improves steadily across self-evolution rounds, rising from 84.29 to 91.10 and largely preserving the consistent growth observed in the PyVision setup. In contrast, R-Zero gains early momentum but oscillates after Round~1. The gap between the two methods widens, suggesting that \methodname{} transfers robustly across harnesses, maintaining its advantage even when the action space changes from dynamic Python execution to the image-zoom-in tool.

\paragraph{Challenger Performance on Visual Reasoning Tasks.}
Despite the promising performance of Marvo, a natural question arises: since the challenger $\pi_{\phi}^{(\mathcal{H})}$ shares the same harness $\mathcal{H}$ as the solver and proposes the question and answer, why do we need to train the solver at all? We evaluate the challenger's per-round answering accuracy on the benchmarks (Table~\ref{tab:challenger-bench}) and find that it is consistently worse than the trained solver's and essentially flat across rounds. The challenger's performance oscillates around on TIR-Bench and improving only marginally on V*. 

The challenger's training signal comes from its role as a task proposer, so its capacity is allocated to synthesizing diverse, verifiable, and appropriately difficult questions rather than to solving itself. Because a proposed question is sampled from the same trajectory as its answer, its phrasing tends to encode the solution path---an exposure bias that inflates self-reported competence, whereas benchmark questions are exogenous and neutral.

\begin{table}[t]
\centering\small
\begin{minipage}{0.6\textwidth}
\centering
\captionof{table}{\textbf{Performance evolution} on V* with Zoom-In tool.}
\label{tab:harness-zoomin}
\begin{tabular}{@{}lllllll@{}}
\toprule
Round                          & Qwen3-VL-4B & 1    & 2    & 3    & 4    & 5    \\ \midrule
\methodname & 84.3        & 84.8 & 85.3 & 89.5 & 89.0 & \textbf{91.1} \\
R-Zero                         & 84.3        & 86.9 & 85.9 & \textbf{89.0} & 88.5 & 86.9 \\ \bottomrule
\end{tabular}
\end{minipage}%
\hfill
\begin{minipage}{0.38\textwidth}
\centering\small
\captionof{table}{
% \color{red}
\bf Challenger performance.
}
\label{tab:challenger-bench}
\begin{tabular}{@{}ccc@{}}
\toprule
\methodname{} Round      & TIR-Bench & V*   \\ \midrule
1 & 19.8      & 71.7 \\
2 & 19.5      & 72.3 \\
3 & 19.8      & 73.8 \\ \bottomrule
\end{tabular}
\end{minipage}
\end{table}

\paragraph{Extension of Margin Reward.} %tbd
The accuracy-margin reward can be generalized to compare the target agent against a \emph{weaker reference agent}, rather than against the model's own baseline:
\begin{equation}
\label{eq:extend-acc-mgn}
\tilde r_{\text{acc-mgn}}(y, \pi_\theta; \mathcal H_1, \mathcal H_2) \;=\;
    \pi_{\theta}^{(\mathcal H_1)}(y \mid I, x) - \pi_{\theta}^{(\mathcal H_2)}(y \mid I, x),
\end{equation}
where $\pi_{\theta}^{(\mathcal H_1)}$ is the target agent we seek to optimize and $\pi_{\theta}^{(\mathcal H_2)}$ is a deliberately weaker visual agent operating under a different harness $\mathcal H_2$. Since \eqref{eq:extend-acc-mgn} has the same margin structure as $r_{\text{acc-mgn}}$, the equilibrium analysis in Sec.~\ref{sec:method-margin} remains.

A practical challenge, however, is constructing an agent that is \emph{reliably weaker} than $\mathcal H_1$: naively simplifying the tool configuration does not necessarily degrade performance, since tool diversity may hurt as often as it helps. We therefore restrict the capability of $\mathcal H_1$ by reducing its maximum conversation turns while keeping the tool configuration fixed, which yields a strictly less expressive agent without altering its action space. Specifically, we instantiate $\mathcal H_2$ as a truncated variant of PyVision ($\mathcal H_1$) whose maximum conversation turns is 3.

Table~\ref{tab:trunc3-rounds} reports an additional \methodname{} run under this setting. Although the trained model still improves round by round, its final performance remains below the corresponding results in Tab.~\ref{tab:main-results}.
We attribute the gap to the sparsity of the extended margin: $\tilde r_{\text{acc-mgn}}$ with  favors questions requiring more than 3 turns, so $\tilde r_{\text{acc-mgn}}$
is near zero on most of the batch, and the effective learning signal is much weaker than that of $r_{\text{acc-mgn}}$, which differs from the direct-answering baseline on a far broader set of prompts.

\begin{table}[t]
\centering\small
\caption{\textbf{Per-round solver performance} under \eqref{eq:extend-acc-mgn} setup, where $\mathcal{H}_2$ is a truncated variant of ${\mathcal{H}_1}$. Parenthesized values are the change versus the previous round.
% \textcolor{green!60!black}{green} for gains and \textcolor{red!80!black}{red} for drops.
}
\label{tab:trunc3-rounds}
\begin{tabular}{@{}lllll@{}}
\toprule
& \multicolumn{1}{c}{Agentic} & \multicolumn{1}{c}{Visual Search} & \multicolumn{2}{c}{Mathematical Reasoning} \\
\cmidrule(lr){2-2}\cmidrule(lr){3-3}\cmidrule(lr){4-5}
& TIR-Bench & V* & {MathVision} & WeMath \\
\midrule
Qwen3-4B-SFT & 22.3 & 78.0 & \textbf{38.8} & 57.4 \\
Round 1 & 22.6\up{0.3} & 78.5\up{0.5} & 35.1\down{3.7} & 57.7\up{0.3} \\
Round 2 & 22.4\down{0.2} & \textbf{79.1}\up{0.5} & 35.8\up{0.6} & 59.0\up{1.3} \\
Round 3 & \textbf{23.2}\up{0.7} & 77.5\down{1.6} & 38.0\up{2.1} & \textbf{62.5}\up{3.5} \\
\bottomrule
\end{tabular}
% \vspace{-0.6cm}
\end{table}

\section{Related Work}

\textbf{Reinforcement Learning for Large Language Models.}
Reinforcement learning (RL) with verifiable rewards (RLVR) fine-tunes models against rule-based checkers in domains with objective correctness criteria such as mathematics and code \citep{shao2024deepseekmath, DBLP:conf/nips/YuZZYZYDFLLLLLL25, zhao2026geometric, DBLP:journals/corr/abs-2507-18071}. 
%paired increasingly
A common prerequisite of RLVR is a large corpus of human-curated tasks with gold labels, which is costly to produce and expensive to scale \citep{su2025crossing}.
To relax the dependency of human annotation, label-free RL derives reward signals from the model's own outputs \citep{agarwal2026unreasonable, zuo2026ttrl, zhang2026consistent, zhang2026right, zhao2026absolute, zhao2026learning, he2026far}. 
These methods remove the need for gold answers but still presuppose a pre-existing corpus of unlabeled questions, falling short of fully autonomous self-improvement.

%Self-evolving LLMs single typically
\textbf{Self-Evolving Large Language Models.}
This aims to eliminate both human-written tasks and labels by letting models generate, refine, and learn from their own experience~\citep{tao2024survey}. Early instantiations adapt self-play to language modeling: a model acts as multiple roles to bootstrap preference or instruction data without annotation \citep{chen2024self, yuan2024self, wu2025self, huang2026r}. 
Self-challenging agents go further by synthesizing the tasks themselves, 
but rely on external verifiers to guarantee feasibility and reliability, 
confining them to domains with an explicit correctness oracle~\citep{zhao2026absolute, wang2026co, xia2025agent0, zhou2026self, lin2025learning, yue2026dr}.
%on visual tasks additional
Despite of the progress, existing self-evolution methods largely rely on pre-defined difficulty or uncertainty metrics to guide the generation of tasks.
For visual agents that largely depend on tool invocation, the plausibility of these metrics requires to be further elaborated.
In this study, we define the margin reward metric, extending self-evolution beyond text-only reasoning to visual agents, and enabling co-evolution of solvers, challengers and tasks towards higher reasoning capacity in visual domains.

\section{Conclusion}

We proposed \methodname{}, a simple yet effective approach for visual agent self-evolution without using human-annotated data.
%
%\methodname{} automatically constructs a post-training curriculum from unlabeled images using a margin-based reward $r_{\text{acc-mgn}}$, which encourages agents to progressively master tool invocation.
In specific, we introduced a margin reward to the challenger, defined as solver's performance difference with and without access to tools. 
Based on the defined margin reward, we create an adaptive curriculum that encourages the solver to learn effective tool invocation.
Through extensive experiments on challenging benchmarks spanning diverse model architectures and agent harnesses, we validated \methodname{}'s consistently performance gains, with striking contrast to the baseline and the state-of-the-art approaches. 
\methodname{} provides a fresh insight to the development of self-evolving agents for a broad range of multimodal tasks. 

\nocite{qwen38}

\bibliography{iclr2027_conference}

@inproceedings{huang2026r,
  title={R-zero: Self-evolving reasoning llm from zero data},
  author={Huang, Chengsong and Yu, Wenhao and Wang, Xiaoyang and Zhang, Hongming and Li, Zongxia and Li, Ruosen and Huang, Jiaxin and Mi, Haitao and Yu, Dong},
  booktitle={International Conference on Learning Representations},
  volume={2026},
  pages={130770--130790},
  year={2026}
}

@article{zhao2026pyvisionrl,
  title={PyVision-RL: Forging Open Agentic Vision Models via RL.},
  author={Zhao, Shitian and Lin, Shaoheng and Li, Ming and Zhang, Haoquan and Peng, Wenshuo and Zhang, Kaipeng and Wei, Chen},
  journal={arxiv preprint arxiv:2602.20739},
  year={2026},
}

@article{zhao2025pyvision,
  title={Pyvision: Agentic vision with dynamic tooling},
  author={Zhao, Shitian and Zhang, Haoquan and Lin, Shaoheng and Li, Ming and Wu, Qilong and Zhang, Kaipeng and Wei, Chen},
  journal={arXiv preprint arXiv:2507.07998},
  year={2025}
}

@article{li2025tir,
  title={Tir-bench: A comprehensive benchmark for agentic thinking-with-images reasoning},
  author={Li, Ming and Zhong, Jike and Zhao, Shitian and Zhang, Haoquan and Lin, Shaoheng and Lai, Yuxiang and Wei, Chen and Psounis, Konstantinos and Zhang, Kaipeng},
  journal={arXiv preprint arXiv:2511.01833},
  year={2025}
}

@inproceedings{DBLP:conf/cvpr/WuX24a,
  author       = {Penghao Wu and
                  Saining Xie},
  title        = {V*: Guided Visual Search as a Core Mechanism in Multimodal LLMs},
  booktitle    = {{IEEE/CVF} Conference on Computer Vision and Pattern Recognition,
                  {CVPR} 2024, Seattle, WA, USA, June 16-22, 2024},
  pages        = {13084--13094},
  publisher    = {{IEEE}},
  year         = {2024},
}

@inproceedings{DBLP:conf/nips/WangPSLRZZL24,
  author       = {Ke Wang and
                  Junting Pan and
                  Weikang Shi and
                  Zimu Lu and
                  Houxing Ren and
                  Aojun Zhou and
                  Mingjie Zhan and
                  Hongsheng Li},
  title        = {Measuring Multimodal Mathematical Reasoning with MATH-Vision Dataset},
  booktitle    = {NeurIPS},
  year         = {2024},
}

@inproceedings{DBLP:conf/eccv/ZhangJZLGQZLCQGL24,
  author       = {Renrui Zhang and
                  Dongzhi Jiang and
                  Yichi Zhang and
                  Haokun Lin and
                  Ziyu Guo and
                  Pengshuo Qiu and
                  Aojun Zhou and
                  Pan Lu and
                  Kai{-}Wei Chang and
                  Yu Qiao and
                  Peng Gao and
                  Hongsheng Li},
  title        = {{MATHVERSE:} Does Your Multi-modal {LLM} Truly See the Diagrams in
                  Visual Math Problems?},
  booktitle    = {ECCV},
  volume       = {15066},
  pages        = {169--186},
  publisher    = {Springer},
  year         = {2024},
}

@article{qiao2024we,
      title={We-Math: Does Your Large Multimodal Model Achieve Human-like Mathematical Reasoning?},
      author={Qiao, Runqi and Tan, Qiuna and Dong, Guanting and Wu, Minhui and Sun, Chong and Song, Xiaoshuai and GongQue, Zhuoma and Lei, Shanglin and Wei, Zhe and Zhang, Miaoxuan and others},
      journal={arXiv preprint arXiv:2407.01284},
      year={2024}
    }

@misc{RealWorldQA2024,
  title        = {RealWorldQA: Real-World Spatial Understanding Benchmark},
  author       = {{xAI}},
  year         = {2024},
}

@article{hrbench,
      title={Divide, Conquer and Combine: A Training-Free Framework for High-Resolution Image Perception in Multimodal Large Language Models}, 
      author={Wenbin Wang and Liang Ding and Minyan Zeng and Xiabin Zhou and Li Shen and Yong Luo and Dacheng Tao},
      year={2024},
      journal={arXiv preprint},
}

@article{shao2024deepseekmath,
  title={Deepseekmath: Pushing the limits of mathematical reasoning in open language models},
  author={Shao, Zhihong and Wang, Peiyi and Zhu, Qihao and Xu, Runxin and Song, Junxiao and Bi, Xiao and Zhang, Haowei and Zhang, Mingchuan and Li, YK and Wu, Yang and others},
  journal={arXiv preprint arXiv:2402.03300},
  year={2024}
}

@inproceedings{DBLP:conf/nips/YuZZYZYDFLLLLLL25,
  author       = {Qiying Yu and
                  Zheng Zhang and
                  Ruofei Zhu and
                  Yufeng Yuan and
                  Xiaochen Zuo and
                  Yu Yue and
                  Weinan Dai and
                  Tiantian Fan and
                  Gaohong Liu and
                  Juncai Liu and
                  Lingjun Liu and
                  Xin Liu and
                  Haibin Lin and
                  Zhiqi Lin and
                  Bole Ma and
                  Guangming Sheng and
                  Yuxuan Tong and
                  Chi Zhang and
                  Mofan Zhang and
                  Ru Zhang and
                  Wang Zhang and
                  Hang Zhu and
                  Jinhua Zhu and
                  Jiaze Chen and
                  Jiangjie Chen and
                  Chengyi Wang and
                  Hongli Yu and
                  Yuxuan Song and
                  Xiangpeng Wei and
                  Hao Zhou and
                  Jingjing Liu and
                  Wei{-}Ying Ma and
                  Ya{-}Qin Zhang and
                  Lin Yan and
                  Yonghui Wu and
                  Mingxuan Wang},
  title        = {{DAPO:} An Open-Source {LLM} Reinforcement Learning System at Scale},
  booktitle    = {Advances in Neural Information Processing Systems},
  year         = {2025},
}

@inproceedings{zhao2026geometric,
  title={Geometric-mean policy optimization},
  author={Zhao, Yuzhong and Liu, Yue and Liu, Junpeng and Chen, Jingye and Hao, Yaru and Lv, Tengchao and Huang, Shaohan and Cui, Lei and Ye, Qixiang and Wan, Fang and others},
  booktitle={International Conference on Learning Representations},
  volume={2026},
  pages={41354--41374},
  year={2026}
}

@article{DBLP:journals/corr/abs-2507-18071,
  author       = {Chujie Zheng and
                  Shixuan Liu and
                  Mingze Li and
                  Xiong{-}Hui Chen and
                  Bowen Yu and
                  Chang Gao and
                  Kai Dang and
                  Yuqiong Liu and
                  Rui Men and
                  An Yang and
                  Jingren Zhou and
                  Junyang Lin},
  title        = {Group Sequence Policy Optimization},
  journal={arxiv preprint arxiv:2507.18071},
  year         = {2025},
}

@article{su2025crossing,
  title={Crossing the reward bridge: Expanding rl with verifiable rewards across diverse domains},
  author={Su, Yi and Yu, Dian and Song, Linfeng and Li, Juntao and Mi, Haitao and Tu, Zhaopeng and Zhang, Min and Yu, Dong},
  journal={arXiv preprint arXiv:2503.23829},
  year={2025}
}

@article{zhao2026absolute,
  title={Absolute zero: Reinforced self-play reasoning with zero data},
  author={Zhao, Andrew and Wu, Yiran and Wu, Tong and Xu, Quentin and Yue, Yang and Lin, Matthieu and Wang, Shenzhi and Wu, Qingyun and Zheng, Zilong and Huang, Gao},
  journal={Advances in Neural Information Processing Systems},
  volume={38},
  pages={105816--105879},
  year={2026}
}

@article{zuo2026ttrl,
  title={{TTRL}: Test-time reinforcement learning},
  author={Zuo, Yuxin and Zhang, Kaiyan and Sheng, Li and Qu, Shang and Cui, Ganqu and Zhu, Xuekai and Li, Haozhan and Long, Xinwei and Hua, Ermo and Qi, Biqing and others},
  journal={Advances in Neural Information Processing Systems},
  volume={38},
  pages={131459--131483},
  year={2026}
}

@article{zhang2026right,
  title={Right question is already half the answer: Fully unsupervised llm reasoning incentivization},
  author={Zhang, Qingyang and Wu, Haitao and Zhang, Changqing and Zhao, Peilin and Bian, Yatao},
  journal={Advances in neural information processing systems},
  volume={38},
  pages={67345--67372},
  year={2026}
}

@article{zhang2026consistent,
  title={Consistent paths lead to truth: Self-rewarding reinforcement learning for llm reasoning},
  author={Zhang, Kongcheng and Yao, Qi and Liu, Shunyu and Wang, Yingjie and Lai, Baisheng and Ye, Jieping and Song, Mingli and Tao, Dacheng},
  journal={Advances in Neural Information Processing Systems},
  volume={38},
  pages={59849--59887},
  year={2026}
}

@inproceedings{zhao2026learning,
  title={Learning to reason without external rewards},
  author={Zhao, Xuandong and Kang, Zhewei and Feng, Aosong and Levine, Sergey and Song, Dawn},
  booktitle={International Conference on Learning Representations},
  volume={2026},
  pages={2548--2581},
  year={2026}
}

@article{agarwal2026unreasonable,
  title={The unreasonable effectiveness of entropy minimization in llm reasoning},
  author={Agarwal, Shivam and Zhang, Zimin and Yuan, Lifan and Han, Jiawei and Peng, Hao},
  journal={Advances in Neural Information Processing Systems},
  volume={38},
  pages={107150--107180},
  year={2026}
}

@article{tao2024survey,
  title={A survey on self-evolution of large language models},
  author={Tao, Zhengwei and Lin, Ting-En and Chen, Xiancai and Li, Hangyu and Wu, Yuchuan and Li, Yongbin and Jin, Zhi and Huang, Fei and Tao, Dacheng and Zhou, Jingren},
  journal={arXiv preprint arXiv:2404.14387},
  year={2024}
}

@article{Qwen2.5-VL,
  title={Qwen2.5-VL Technical Report},
  author={Bai, Shuai and Chen, Keqin and Liu, Xuejing and Wang, Jialin and Ge, Wenbin and Song, Sibo and Dang, Kai and Wang, Peng and Wang, Shijie and Tang, Jun and Zhong, Humen and Zhu, Yuanzhi and Yang, Mingkun and Li, Zhaohai and Wan, Jianqiang and Wang, Pengfei and Ding, Wei and Fu, Zheren and Xu, Yiheng and Ye, Jiabo and Zhang, Xi and Xie, Tianbao and Cheng, Zesen and Zhang, Hang and Yang, Zhibo and Xu, Haiyang and Lin, Junyang},
  journal={arXiv preprint arXiv:2502.13923},
  year={2025}
}

@article{Qwen3-VL,
      title={Qwen3-VL Technical Report}, 
      author={Shuai Bai and Yuxuan Cai and Ruizhe Chen and Keqin Chen and Xionghui Chen and Zesen Cheng and Lianghao Deng and Wei Ding and Chang Gao and Chunjiang Ge and Wenbin Ge and Zhifang Guo and Qidong Huang and Jie Huang and Fei Huang and Binyuan Hui and Shutong Jiang and Zhaohai Li and Mingsheng Li and Mei Li and Kaixin Li and Zicheng Lin and Junyang Lin and Xuejing Liu and Jiawei Liu and Chenglong Liu and Yang Liu and Dayiheng Liu and Shixuan Liu and Dunjie Lu and Ruilin Luo and Chenxu Lv and Rui Men and Lingchen Meng and Xuancheng Ren and Xingzhang Ren and Sibo Song and Yuchong Sun and Jun Tang and Jianhong Tu and Jianqiang Wan and Peng Wang and Pengfei Wang and Qiuyue Wang and Yuxuan Wang and Tianbao Xie and Yiheng Xu and Haiyang Xu and Jin Xu and Zhibo Yang and Mingkun Yang and Jianxin Yang and An Yang and Bowen Yu and Fei Zhang and Hang Zhang and Xi Zhang and Bo Zheng and Humen Zhong and Jingren Zhou and Fan Zhou and Jing Zhou and Yuanzhi Zhu and Ke Zhu},
	  journal={arXiv preprint arXiv:2511.21631},
      year={2025}
}

@article{chen2024self,
  title={Self-play fine-tuning converts weak language models to strong language models},
  author={Chen, Zixiang and Deng, Yihe and Yuan, Huizhuo and Ji, Kaixuan and Gu, Quanquan},
  journal={arXiv preprint arXiv:2401.01335},
  year={2024}
}

@article{yuan2024self,
  title={Self-rewarding language models},
  author={Yuan, Weizhe and Pang, Richard Yuanzhe and Cho, Kyunghyun and Li, Xian and Sukhbaatar, Sainbayar and Xu, Jing and Weston, Jason},
  journal={arXiv preprint arXiv:2401.10020},
  year={2024}
}

@inproceedings{wu2025self,
  title={Self-play preference optimization for language model alignment},
  author={Wu, Yue and Sun, Zhiqing and Hughes, Rina and Ji, Kaixuan and Yang, Yiming and Gu, Quanquan},
  booktitle={International Conference on Learning Representations},
  volume={2025},
  pages={91558--91582},
  year={2025}
}

@article{xia2025agent0,
  title={Agent0: Unleashing self-evolving agents from zero data via tool-integrated reasoning},
  author={Xia, Peng and Zeng, Kaide and Liu, Jiaqi and Qin, Can and Wu, Fang and Zhou, Yiyang and Xiong, Caiming and Yao, Huaxiu},
  journal={arXiv preprint arXiv:2511.16043},
  year={2025}
}

@article{wang2026co,
  title={Co-evolving llm coder and unit tester via reinforcement learning},
  author={Wang, Yinjie and Yang, Ling and Tian, Ye and Shen, Ke and Wang, Mengdi},
  journal={Advances in Neural Information Processing Systems},
  volume={38},
  pages={143630--143664},
  year={2026}
}

@article{zhou2026self,
  title={Self-challenging language model agents},
  author={Zhou, Yifei and Levine, Sergey and Weston, Jason and Li, Xian and Sukhbaatar, Sainbayar},
  journal={Advances in Neural Information Processing Systems},
  volume={38},
  pages={113959--113991},
  year={2026}
}

@article{lin2025learning,
  title={Learning to solve and verify: A self-play framework for code and test generation},
  author={Lin, Zi and Shen, Sheng and Kulikov, Ilia and Shang, Jingbo and Weston, Jason and Nie, Yixin},
  journal={arXiv preprint arXiv:2502.14948},
  year={2025}
}

@article{yue2026dr,
  title={Dr. zero: Self-evolving search agents without training data},
  author={Yue, Zhenrui and Upasani, Kartikeya and Yang, Xianjun and Ge, Suyu and Nie, Shaoliang and Mao, Yuning and Liu, Zhe and Wang, Dong},
  journal={arXiv preprint arXiv:2601.07055},
  year={2026}
}

@article{comanici2025gemini,
  title={Gemini 2.5: Pushing the frontier with advanced reasoning, multimodality, long context, and next generation agentic capabilities},
  author={Comanici, Gheorghe and Bieber, Eric and Schaekermann, Mike and Pasupat, Ice and Sachdeva, Noveen and Dhillon, Inderjit and Blistein, Marcel and Ram, Ori and Zhang, Dan and Rosen, Evan and others},
  journal={arXiv preprint arXiv:2507.06261},
  year={2025}
}

@article{liu2025agent0,
  title={Agent0-vl: Exploring self-evolving agent for tool-integrated vision-language reasoning},
  author={Liu, Jiaqi and Xiong, Kaiwen and Xia, Peng and Zhou, Yiyang and Ji, Haonian and Feng, Lu and Han, Siwei and Ding, Mingyu and Yao, Huaxiu},
  journal={arXiv preprint arXiv:2511.19900},
  year={2025}
}

@inproceedings{hong2026deepeyesv2,
  title={Deepeyesv2: Toward agentic multimodal model},
  author={Hong, Jack and Zhao, Chenxiao and Zhu, ChengLin and Lu, Weiheng and Xu, Guohai},
  booktitle={International Conference on Learning Representations},
  volume={2026},
  pages={114851--114872},
  year={2026}
}

@misc{o3,
    author = {OpenAI},
    title = {Introducing o3 and o4-mini},
    year = {2025}
}

@inproceedings{zheng2026deepeyes,
  title={Deepeyes: Incentivizing" thinking with images" via reinforcement learning},
  author={Zheng, Ziwei and Yang, Minghao and Hong, Jack and Zhao, Chenxiao and Xu, Guohai and Yang, Le and Shen, Chao},
  booktitle={International Conference on Learning Representations},
  volume={2026},
  pages={126775--126798},
  year={2026}
}

@misc{qwen38,
    title = {{Qwen3.8-Max}: A New Bar for Coding and Cowork},
    author = {{Qwen Team}},
    month = {August},
    year = {2026}
}

@inproceedings{he2026far,
  title={How Far Can Unsupervised RLVR Scale LLM Training?},
  author={He, Bingxiang and Zuo, Yuxin and Liu, Zeyuan and Zhao, Shangziqi and Fu, Zixuan and Yang, Junlin and Qian, Cheng and Zhang, Kaiyan and Fan, Yuchen and Cui, Ganqu and others},
  booktitle={International Conference on Learning Representations},
  volume={2026},
  pages={14823--14865},
  year={2026}
}
\bibliographystyle{iclr2027_conference}

\newpage
\appendix

\section{LLM and AI Usage Statement}

In this work, we used generative AI tools for improving language clarity and grammar, writing utility scripts for data preprocessing and postprocessing, debugging, and restructuring the code base; core algorithmic implementation was done and reviewed by the authors. 
We have not used generative AI tools for experimental design, or drafting scientific arguments, and figure generation is not applicable to this work. 
Outputs from the AI-assisted tools are  meticulously reviewed, revised, and verified by the authors, who retain full responsibility for all content presented in this paper.
\section{Preliminaries}
\label{app:preliminaries}

\subsection{Visual Agents}
\label{app:visual-agents}

Visual agents extend LLMs with a closed observe–reason–act loop over images: rather than answering in a single pass, they iteratively acquire visual evidence before responding. In this study, we mainly consider the following two types of visual agents, as shown in Fig.~\ref{fig:visual-agent-flowchart}.

\paragraph{Zoom-in–only agents} (e.g., DeepEyes~\citep{zheng2026deepeyes}). These agents treat high-resolution understanding as sequential search. Given a downsampled full image, the agent issues zoom-in actions that crop and enlarge regions from the original resolution, progressively resolving fine details (small text, dense charts, tiny UI elements). The minimal action space (zoom-in + answer) makes the policy simple and robust, but offers no computational manipulation — all evidence must be located via spatial attention alone.

\paragraph{Code-execution agents} (e.g., PyVision~\citealp{zhao2025pyvision}). Here the action space is a Python interpreter: the agent writes code to crop regions, run OCR, detect contours, sample pixels, or extract chart values, computing visual evidence rather than merely viewing it. This enables precise quantitative answers but risks buggy code, fragile heuristics, and long unproductive traces, complicating control and termination.

\begin{figure}[t]
    \centering
    \includegraphics[width=1\linewidth]{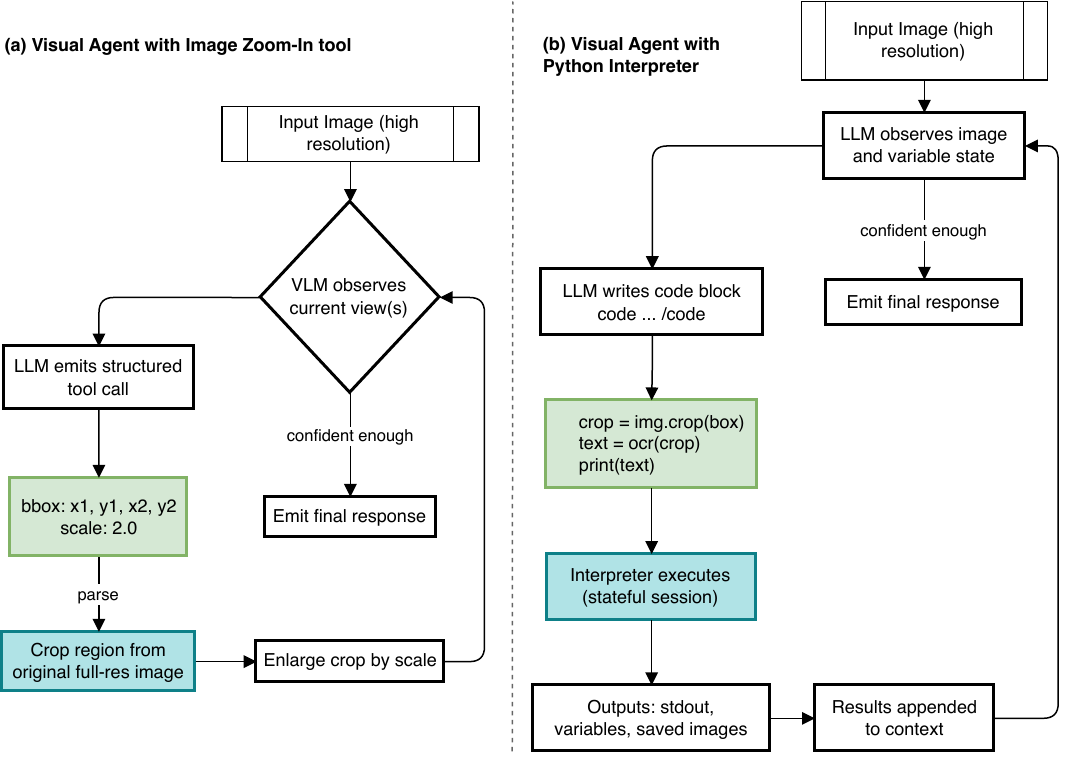}
    \caption{Visual Agent Flowcharts}
    \label{fig:visual-agent-flowchart}
\end{figure}

\subsection{Reinforcement Learning with GRPO}
\label{app:grpo}

Both agents are trained with Group Relative Policy Optimization
(GRPO)~\citep{shao2024deepseekmath}. For each prompt $q$, the behavior policy
$\pi_{\theta_{\mathrm{old}}}$ samples a group of $G$ responses
$\{o_1, \dots, o_G\}$ with scalar rewards $\{r_1, \dots, r_G\}$. In place of a
learned value function, GRPO uses the group mean as a baseline, normalizing
each reward within its group:
\begin{equation}
\hat{A}_i \;=\; \frac{r_i - \operatorname{mean}(\{r_j\}_{j=1}^{G})}
                    {\operatorname{std}(\{r_j\}_{j=1}^{G})},
\label{eq:grpo-advantage}
\end{equation}
and maximizes a clipped surrogate over all response tokens:
\begin{equation}
\mathcal{J}(\theta) = \mathbb{E}\!\left[
\frac{1}{G}\sum_{i=1}^{G} \frac{1}{|o_i|} \sum_{t=1}^{|o_i|}
\min\!\Big(
\rho_{i,t}\,\hat{A}_i,\;
\operatorname{clip}\big(\rho_{i,t},\, 1-\varepsilon_{\mathrm{lo}},\,
1+\varepsilon_{\mathrm{hi}}\big)\,\hat{A}_i
\Big)\right],
\label{eq:grpo-objective}
\end{equation}
where $\rho_{i,t} = \pi_{\theta}(o_{i,t} \mid q, o_{i,<t}) \,/\,
\pi_{\theta_{\mathrm{old}}}(o_{i,t} \mid q, o_{i,<t})$ is the per-token
importance ratio against the behavior policy. We use asymmetric clipping
bounds $\varepsilon_{\mathrm{lo}} = 0.2$, $\varepsilon_{\mathrm{hi}} = 0.28$
and no KL penalty against a reference model.

\subsection{Asynchronous Training}
\label{app:async-training}

In synchronous on-policy training, each batch is generated by the current
policy, $\pi_{\theta_{\mathrm{old}}} = \pi_{\theta}$, and the importance
ratios in ~\eqref{eq:grpo-objective} start at $1$. In the
\emph{asynchronous} variant used for Solver training, generation of the next
batch overlaps the gradient step on the current one, so each batch is produced
by a policy that lags the learner by one step: for the batch consumed at
learner step $\theta_t$,
\begin{equation}
\pi_{\theta_{\mathrm{old}}} \;=\; \pi_{\theta_{t-1}},
\qquad
\rho_{i,t} \;=\;
\frac{\pi_{\theta_t}(o_{i,t} \mid q, o_{i,<t})}
     {\pi_{\theta_{t-1}}(o_{i,t} \mid q, o_{i,<t})}.
\label{eq:async-ratio}
\end{equation}
The objective is otherwise unchanged: the importance ratio in \eqref{eq:grpo-objective} corrects for the one-step staleness exactly,
and the clip bounds the deviation it can introduce. The group-relative
advantage of \eqref{eq:grpo-advantage} is unaffected, since it depends
only on rewards within a batch, all sampled from the same stale policy.

\section{Nash Equilibrium Analysis}
\label{app:nash}

We formalize the equilibrium characterization stated in Section~\ref{sec:equilibrium}. For generality, we use the extended form (\eqref{eq:extend-acc-mgn}) of the accuracy margin to develop the analysis.
Throughout, we regard the alternating optimization in \eqref{eq:objection-general-form} as a two-player game between the challenger and the solver.

\paragraph{Setup.}
For notational simplicity, let $q=(I,x,y)$ denote a proposed problem, and define
\begin{equation}
    p_i(q;\theta)
    :=
    \pi_{\theta}^{(\mathcal H_i)}(y\mid I,x),
    \qquad i\in\{1,2\},
\end{equation}
together with the harness margin
\begin{equation}
    \Delta_\theta(q)
    :=
    p_1(q;\theta)-p_2(q;\theta).
\end{equation}
Let $\rho_\phi$ denote the joint distribution over $q$ induced by $I\sim\mathcal I$ and $(x,y)\sim\pi_\phi(\cdot\mid I)$.

\paragraph{Two-Player Formulation.}
Conditioned on a proposed pair $(I,x,y)$, the expected solver accuracy is
\begin{equation}
    \mathbb E_{\hat y\sim\pi_\theta^{(\mathcal H_1)}(\cdot\mid I,x)}
    \left[\mathbf 1_{\{\hat y=y\}}\right]
    =
    p_1(q;\theta).
\end{equation}
Hence the challenger and solver payoffs can be written as
\begin{equation}
    \label{eq:app-payoff}
    U_{\mathrm C}(\phi,\theta)
    =
    \mathbb E_{q\sim\rho_\phi}
    \left[\Delta_\theta(q)\right],
    \quad \text{and} \quad
    U_{\mathrm S}(\phi,\theta)
    =
    \mathbb E_{q\sim\rho_\phi}
    \left[p_1(q;\theta)\right],
\end{equation}
respectively. A pair $(\phi^*,\theta^*)$ is a Nash equilibrium if
\begin{equation}
    U_{\mathrm C}(\phi^*,\theta^*)
    \geq
    U_{\mathrm C}(\phi,\theta^*)
    \quad \text{ and } \quad
    U_{\mathrm S}(\phi^*,\theta^*)
    \geq
    U_{\mathrm S}(\phi^*,\theta)
    \quad\text{for all admissible }\phi,\theta.
\end{equation}
% For nonconvex policies, the same conditions can be interpreted locally, in which case the results below characterize stable local Nash equilibria.

% \paragraph{Support optimality of the challenger.}
% The challenger payoff \eqref{eq:app-challenger-payoff} is linear in $\rho_\phi$. Consequently, at equilibrium all problems receiving positive challenger probability must attain the same maximal margin:
% \begin{equation}
%     \label{eq:app-support-optimality}
%     q\in\operatorname{supp}(\rho_{\phi^*})
%     \quad\Longrightarrow\quad
%     \Delta_{\theta^*}(q)
%     =
%     m^*
%     :=
%     \max_{q'}\Delta_{\theta^*}(q').
% \end{equation}
% Otherwise, the challenger could strictly improve its payoff by moving probability mass from a lower-margin problem to a higher-margin problem. 
% Thus, the equilibrium support is contained in the current margin-maximizing problem frontier.

\begin{proposition}[Support optimality of the challenger]
\label{prop:support-optimality}
Fix the solver parameters $\theta$, and let $\mathcal Q$ denote the feasible set of proposed problems $q=(I,x,y)$. Assume that the harness margin $\Delta_\theta:\mathcal Q\to\mathbb R$ is bounded and that
\begin{equation}
    m_\theta^*
    :=
    \max_{q'\in\mathcal Q}\Delta_\theta(q')
\end{equation}
is attained. Then a challenger distribution $\rho_\phi$ is a best response to $\theta$ if and only if
\begin{equation}
    \label{eq:support-optimality-a.s.}
    \rho_\phi\!\left(
    \left\{
    q\in\mathcal Q:
    \Delta_\theta(q)=m_\theta^*
    \right\}
    \right)=1.
\end{equation}
% Equivalently,
% \begin{equation}
%     \label{eq:support-optimality-supp}
%     q\in\operatorname{supp}(\rho_\phi)
%     \quad\Longrightarrow\quad
%     \Delta_\theta(q)=m_\theta^*.
% \end{equation}
% At a Nash equilibrium $(\phi^*,\theta^*)$, condition \eqref{eq:support-optimality-supp} holds with $\theta=\theta^*$ and $\rho_\phi=\rho_{\phi^*}$.
\end{proposition}

\begin{proof}
For fixed $\theta$, the challenger payoff is $U_{\mathrm C} (\phi, \theta) = \mathbb E_{q\sim\rho_\phi}[\Delta_\theta(q)]$.
Since $\Delta_\theta(q)\leq m_\theta^*$ for every $q$, we immediately have $U_{\mathrm C}(\phi,\theta)\leq m_\theta^*$.
If $\rho_\phi$ is supported on
$\{q:\Delta_\theta(q)=m_\theta^*\}$, then
$U_{\mathrm C}(\phi,\theta)=m_\theta^*$, so $\rho_\phi$ is a best response.

Conversely, suppose that $\rho_\phi$ assigns positive probability to the suboptimal set
\begin{equation}
    A
    :=
    \{q\in\mathcal Q:\Delta_\theta(q)<m_\theta^*\}.
\end{equation}
Then
\begin{equation}
\begin{aligned}
U_{\mathrm C}(\phi,\theta)
&=
\int_A\Delta_\theta(q)\,d\rho_\phi(q)
+
\int_{A^c}\Delta_\theta(q)\,d\rho_\phi(q) \\
& < m_\theta^*\rho_\phi(A) + m_\theta^*\rho_\phi(A^c) \\
&=
m_\theta^*.
\end{aligned}
\end{equation}
Thus $\rho_\phi$ is not a best response. Therefore, a best response must satisfy \eqref{eq:support-optimality-a.s.}.
\end{proof}

\paragraph{Solver regularity.}
% To exclude plateau equilibria at intermediate accuracy, we use the following local responsiveness condition: for every equilibrium-support problem $q$ with $p_1(q;\theta)<1$, there exists an admissible solver update direction $v$ such that
% \begin{equation}
%     \label{eq:app-local-responsiveness}
%     D_v p_1(q;\theta)>0,
%     \qquad
%     D_v p_1(q';\theta)\geq 0
%     \quad
%     \text{for all other }q'\in\operatorname{supp}(\rho_\phi),
% \end{equation}
% where $D_v$ denotes the directional derivative along $v$. In words, an unsaturated problem on the challenger support can be locally improved without reducing the solver's accuracy on the remaining support. This condition holds, for example, when the relevant answer logit is locally controllable and the model is sufficiently expressive on the induced problem distribution. It rules out parameter plateaus in which a problem remains partially solved only because no feasible solver update can increase its accuracy.

Let $\mathcal A_{\theta} := \{q: p_1( q;\theta) < 1\}$ be the set of unsaturated problems of solver $\pi_\theta^{(\mathcal H_1)}$. 
To exclude plateau equilibria at intermediate accuracy, we use the following local regularity condition: if $\rho_\theta(\mathcal A_\theta) > 0$, then there exists some update vector $v$ and some $\varepsilon > 0$ such that
\begin{equation}
    \label{eq:app-local-responsiveness}
    \mathbb E_{q\sim\rho_{\phi}}\left[ p_1( q;\theta + \varepsilon v) - p_1( q;\theta) \right] > 0.
\end{equation}
If $p_1$ is differentiable in $\theta$, then it is equivalent to 
\[
\mathbb{E}_{q\sim\rho_{\phi}}\left[ D_vp_1(q;\theta) \right] > 0  \quad\text{for some $v$,}
\]
where $D_v$ denotes directional derivative.
We propose this condition to rule out parameter plateaus within $\Theta$ in which a problem remains partially solved only because no feasible solver update can increase its accuracy.

\begin{proposition}
% [binary saturation and the equilibrium dichotomy]
    Let $(\phi^*,\theta^*)$ be a Nash equilibrium satisfying Prop.~\ref{prop:support-optimality} and the local regularity condition \eqref{eq:app-local-responsiveness}. Then, for almost every $q\sim\rho_{\phi^*}$,
\begin{equation}
    \label{eq:app-target-saturation}
    p_1(q;\theta^*)=1.
\end{equation}
Moreover, exactly one of the following two cases holds on the equilibrium support:
\begin{enumerate}
    \item \textbf{Gap-closed equilibrium.} If $m^*=0$, then
    \begin{equation}
        p_1(q;\theta^*)=p_2(q;\theta^*)=1.
    \end{equation}
    The two harnesses are indistinguishable on the equilibrium problem distribution. In particular, the orchestration behavior exercised by $\mathcal H_1$ on these problems has been transferred to $\mathcal H_2$ through the shared solver weights.

    \item \textbf{Gap-saturated equilibrium.} If $m^*>0$, then
    \begin{equation}
        p_1(q;\theta^*)=1,
        \qquad
        p_2(q;\theta^*)=1-m^*<1.
    \end{equation}
    The solver has saturated its accuracy under $\mathcal H_1$, while the same weights remain inaccurate under $\mathcal H_2$. The residual margin therefore identifies a harness-specific capability gap.
\end{enumerate}
In particular, no positive-mass equilibrium problem satisfies
\begin{equation}
    0<p_1(q;\theta^*)<1.
\end{equation}
\end{proposition}

\begin{proof}
Suppose, for contradiction, that there is a positive-measure set of equilibrium-support problems with $p_1(q;\theta^*)<1$. 
By \eqref{eq:app-local-responsiveness}, there exists an admissible solver update that strictly increases $\mathbb E_q[p_1]$. This update increases $U_{\mathrm S}(\phi^*,\theta)$, contradicting the assumption that $\theta^*$ is a best response to $U_{\mathrm S}(\phi^*,\theta)$. 
Hence \eqref{eq:app-target-saturation} holds.

Proposition~\ref{prop:support-optimality} gives
$
\Delta_{\theta^*}(q) = 1-p_2(q;\theta^*) = m^*
$
throughout the equilibrium support. If $m^*=0$, then $p_2(q;\theta^*)=1$, yielding the gap-closed case. If $m^*>0$, then $p_2(q;\theta^*)=1-m^*<1$, yielding the gap-saturated case. Since probabilities are bounded above by one, no negative-margin case can occur once $p_1$ is saturated.
\end{proof}

\section{Implementation Details}

\subsection{Base Models}
\label{sec:base-models}

To demonstrate that {\bf \methodname} generalizes beyond a specific model family or scale, we instantiate our framework with two base models that differ substantially in architecture, training pipeline, and capacity: Qwen2.5-VL-7B~\citep{Qwen2.5-VL} and Qwen3-VL-4B~\citep{Qwen3-VL}.

Although both models follow the three-module VLM paradigm (vision encoder, vision-language merger, and LLM), their designs diverge in nearly every component: Qwen2.5-VL employs a ViT trained from scratch with window attention, compresses every $2\times2$ group of adjacent patch features into a single token via an MLP merger injected only at the LLM input, and uses MRoPE with contiguous temporal/height/width sections aligned to absolute time; in contrast, Qwen3-VL adopts a pretrained SigLIP-2 encoder continually trained at dynamic resolutions, introduces DeepStack to inject multi-level ViT features into the first several LLM layers, and uses Interleaved-MRoPE together with explicit text-based timestamps for temporal grounding.
The LLM backbones also belong to different generations (Qwen2.5 vs.\ Qwen3) with distinct pretraining corpora and post-training recipes, and the two models differ considerably in capacity (7B vs.\ 4B).

Importantly, {\bf \methodname} imposes no requirement on the base model beyond the ability to invoke tools and produce verifiable answers---the challenger--solver co-evolution, the accuracy-margin reward, and the GRPO-based updates are all model-agnostic---and as shown in Sec.~\ref{sec:experiment}, both models achieve consistent improvements after training, indicating that the gains stem from the paradigm itself rather than from idiosyncrasies of a particular backbone.

\subsection{SFT Details}
\label{app:sft-training-details}

For the PyVision harness, SFT checkpoints serve as the cold-start for self-evolution, as the base model struggles to operate the code interpreter directly. For Qwen2.5-VL-7B, an SFT checkpoint is readily provided by \citet{zhao2026pyvisionrl}; for Qwen3-VL-4B, the base model is instead finetuned on the same SFT data. The cold-start stage is lightweight, aiming only to familiarize the model with the harness.

\begin{figure}[h]
\centering
\begin{minipage}[t]{0.45\columnwidth}
\centering\small
\begin{tabular}{ll}
\toprule
\multicolumn{2}{l}{\emph{SFT configuration}} \\
\midrule
Base model & Qwen3-VL-4B-Instruct \\
Epochs & 1 (shuffled) \\
Global batch size & 16 \\
Micro batch size & 1 \\
Steps & 430 (6{,}880 episodes) \\
Learning rate & $1\times10^{-5}$ \\
LR schedule & cosine, $\min$ $1\times10^{-6}$ \\
Warmup fraction & 0.1 \\
Optimizer & Adam ($\beta_1{=}0.9$, $\beta_2{=}0.95$) \\
Weight decay & 0.1 \\
Sequence length & 16{,}384 \\
Vision tower & frozen \\
\bottomrule
\end{tabular}
\end{minipage}\hfill
\begin{minipage}[c]{0.52\columnwidth}
\centering
\includegraphics[width=\linewidth]{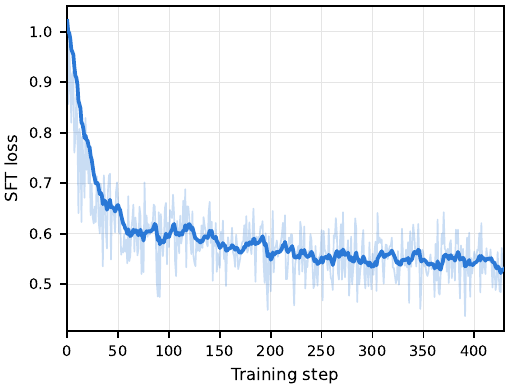}
\end{minipage}
\caption{\textbf{Left:} SFT hyperparameters. \textbf{Right:} SFT training loss.}
\label{fig:sft-config-loss}
\end{figure}

\subsection{RL Hyperparameters}
\label{app:training-details}

Table~\ref{tab:hyperparameters} lists the hyperparameters of the
accuracy-margin run summarized by Algorithm~\ref{alg:accuracy-margin}. A skeleton code base is enclosed in the supplementary.

\begin{table}[h]
\centering\small
\caption{Training Hyperparameters.}
\label{tab:hyperparameters}
\begin{tabular}{ll}
\toprule
 & Value \\
\midrule
\multicolumn{2}{l}{\emph{Run / data}} \\
Rounds $R$ & 3 \\
Images per round $B$ & 1{,}280 (cyclic disjoint slices; pool 44{,}643, seed 42) \\
Image long-edge cap & 1{,}536 pixels \\
\midrule
\multicolumn{2}{l}{\emph{Challenger (GRPO)}} \\
Steps per round $m$ & 10 \\
Rollout batch size & 32 \\
Group size $G$ & 4 \\
Learning rate & $1\times10^{-6}$ \\
Rollout temperature & 1.0 \\
Max response length & 4{,}096 tokens \\
Code-turn budget & 8 turns \\
Sandbox timeout & 30 s \\
Reward samples $k_r$ & 1 (temperature 1.0, both harnesses) \\
Parse-failure reward & $-1.0$ \\
\midrule
\multicolumn{2}{l}{\emph{Data filter}} \\
Filter solves $k_f$ & 4 (temperature 1.0) \\
% Pass-rate band $[\tau_{\mathrm{lo}}, \tau_{\mathrm{hi}}]$ & $[0.25, 0.75]$ \\
Candidates per image & 4 \\
Kept per round & 1{,}280 (cap) \\
\midrule
\multicolumn{2}{l}{\emph{Solver (GRPO)}} \\
Steps per round $n$ & 20 (one pass; no epochs) \\
Rollout batch size & 64 \\
Group size $G$ & 8 \\
Learning rate & $1\times10^{-6}$ \\
Rollout temperature & 1.0 \\
Max response length & 8{,}192 tokens \\
Max images per episode & 20 \\
Tool-use bonus & $+0.2$ \\
Format penalty & $-0.5$ \\
\midrule
\multicolumn{2}{l}{\emph{Evaluation}} \\
Sampling & $t{=}0.6$, top-$p{=}1.0$, top-$k{=}20$, min-$p{=}0$ \\
Max tokens & 16{,}000 \\
Max images per episode & 20 \\
\bottomrule
\end{tabular}
\end{table}

\subsection{Evaluation Benchmarks}
\label{sec:appendix-benchmarks}

We evaluate the solver on seven benchmarks spanning three families:
agentic thinking-with-images, high-resolution visual search, and
multimodal mathematical reasoning. Unless stated otherwise, we report
accuracy with greedy decoding, and we run the solver in its full agentic
configuration (i.e., with access to the image zoom-in tool).

\paragraph{TIRBench.}
TIRBench~\citep{li2025tir} is a benchmark for agentic thinking-with-images
reasoning. It comprises 13 tasks and 1{,}215 samples, each designed to be
unsolvable by static observation alone and instead requiring multi-step,
tool-based image manipulation---e.g., zooming, rotating, contrast
enhancement, drawing auxiliary lines, or segmenting and reassembling
pieces. Because TIRBench directly measures tool-orchestration capability,
it is the primary in-domain benchmark for our method.

\paragraph{HRBench.}
HRBench~\citep{hrbench} targets fine-grained visual perception and spatial
reasoning over ultra-high-resolution images (the HRBench-4K subset), where
the decisive visual evidence occupies only a small region of a very large
image. It stresses precisely the regime in which a zoom-in tool should
help most, and in which naive resizing destroys the information needed to
answer.

\paragraph{V$^*$ Bench.}
V$^*$Bench~\citep{DBLP:conf/cvpr/WuX24a} evaluates visual search over high-resolution,
visually crowded images (191 samples, average resolution
$2246 \times 1582$). It contains two subtasks: attribute recognition
(115 samples), which asks for a specific attribute (e.g., color,
material) of an object in the scene, and spatial relationship reasoning
(76 samples), which asks for the relative position of two objects. Both
require the model to localize small, task-relevant regions before
answering.

\paragraph{RealWorldQA.}
RealWorldQA~\citep{RealWorldQA2024} consists of over 700 real-world
photographs---largely captured from vehicles and other everyday
settings---each paired with a question and a short, easily verifiable
answer. It measures practical spatial understanding and grounding in
naturalistic scenes.

\paragraph{MathVision.}
MathVision~\citep{DBLP:conf/nips/WangPSLRZZL24} contains 3{,}040 competition-level
mathematics problems embedded in visual contexts (figures, diagrams,
charts), spanning 16 mathematical disciplines across 5 difficulty
levels. Problems are drawn from real high-school and university
competitions, emphasizing deep reasoning over diagrams rather than
perception alone.

\paragraph{MathVerse.}
MathVerse~\citep{DBLP:conf/eccv/ZhangJZLGQZLCQGL24} is an all-around visual math benchmark of
2{,}612 high-quality multi-subject problems with diagrams. Each problem
is human-transformed into six distinct versions that vary the degree of
information provided in text versus vision, yielding 15K test instances
in total. This design tests whether a model truly interprets the diagram
rather than shortcutting from the textual question; we report results
under the CoT evaluation protocol.

\paragraph{WeMath.}
WeMath~\citep{qiao2024we} is a large-scale benchmark of visual math
problems organized into 67 hierarchical knowledge concepts. Rather than
scoring end-to-end correctness only, it decomposes problems into
individual knowledge steps, enabling a fine-grained, process-oriented
diagnosis of mathematical reasoning in visual contexts.

\subsection{Training Data}
\label{sec:appendix-training-data}

We extract images from the RL training dataset from PyVision-RL~\citep{zhao2026pyvisionrl} and use these unlabeled images as the only inputs throughout \methodname{} experiments. Due to computation limitation, we filter out over-sized images, resulting a pool of 42,124 images spanning various sources and resolution scales (Fig.~\ref{fig:resolution}).

\begin{figure}[h]
    \centering
    \includegraphics[width=0.75\linewidth]{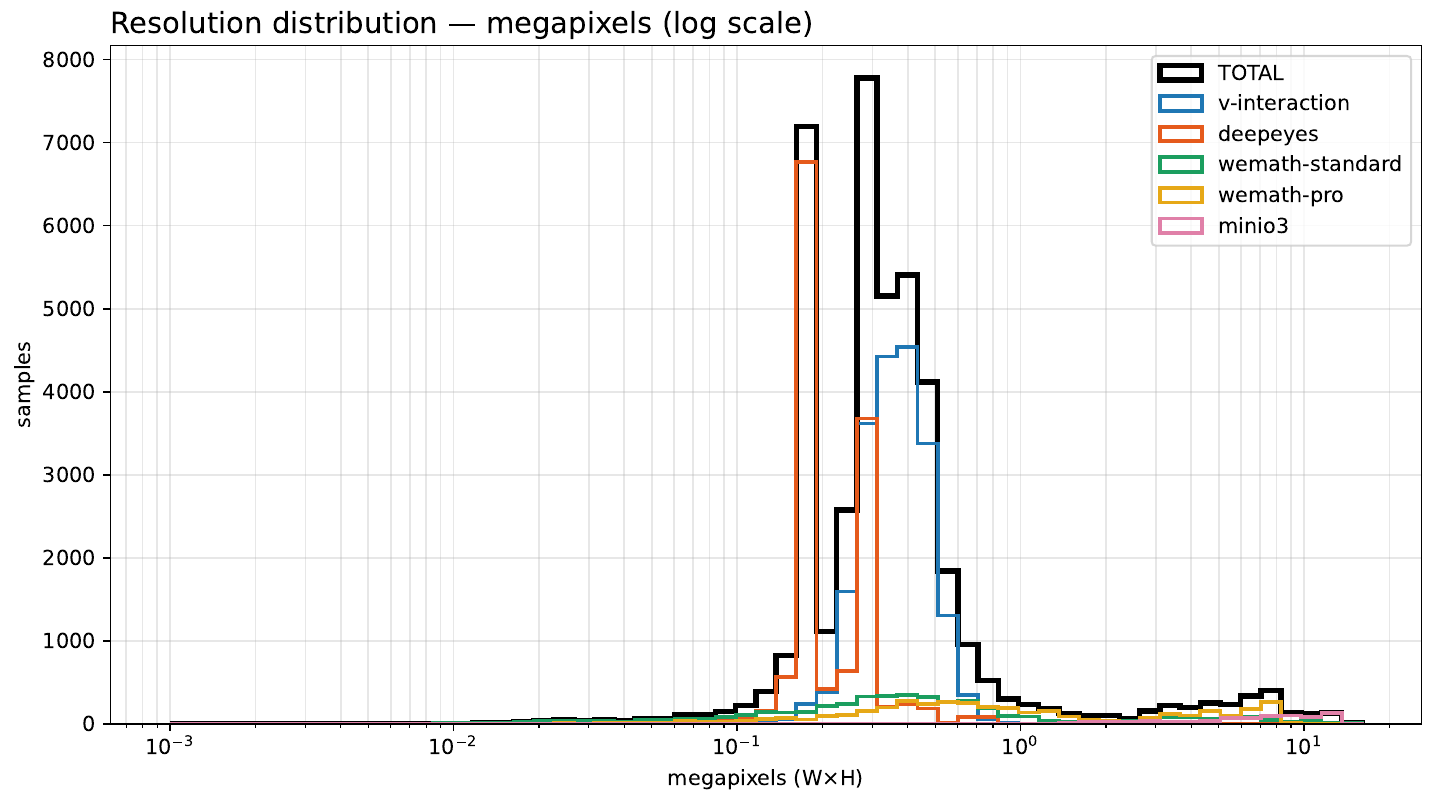}
    \caption{Resolution Distribution of Training Image Pool.}
    \label{fig:resolution}
\end{figure}

\section{Challenger Prompt}
\label{app:challenger-prompt}

Listing~\ref{lst:challenger-prompt} shows the complete prompt of Challenger with the PyVision harness~\cite{zhao2025pyvision}. Following the PyVision-RL wire format, the
episode is a \emph{single user message with no system prompt}: the image is
embedded in \texttt{<image\_clue\_0>} tags, and the \texttt{\{width\}/\{height\}} placeholders are filled with the image's pixel dimensions.
The template re-points the reference solver scaffold at question
\emph{authoring}: the finalize contract is a single fenced \texttt{json} block
(question, four options, gold letter) in place of the solver's boxed answer,
and the prompt requires the model to inspect at least one cropped region with
code before authoring.

\begin{tcolorbox}[enhanced, breakable, sharp corners,
                  colback=black!4, colframe=black!75, boxrule=0.4pt,
                  boxsep=0pt, left=6pt, right=6pt, top=5pt, bottom=5pt,
                  pad at break*=3mm]
\begin{lstlisting}[
  basicstyle=\ttfamily\footnotesize,
  columns=fullflexible, keepspaces=true, breaklines=true,
  captionpos=t, abovecaptionskip=6pt, belowcaptionskip=6pt,
  caption={The Challenger Prompt.},
  label={lst:challenger-prompt},
  literate={’}{'}1 {—}{{\textemdash}}1]
<image_clue_0><image></image_clue_0>

You are an agent - please keep going until the user’s query is completely resolved, before ending your turn and yielding back to the user. Only terminate your turn when you are sure that the problem is solved.

Your task: author **exactly one** challenging multiple-choice question (four options A-D) about the image. The question must require fine-grained visual examination of the image to answer correctly, must be answerable purely from the image (no outside knowledge), and the gold answer must be unambiguous.

You now have the ability to selectively write executable Python code to enhance your inspection of the image. The Python code will be executed by an external sandbox.

You MUST plan extensively before each function call, and reflect extensively on the outcomes of the previous function calls. DO NOT do this entire process by making function calls only, as this can impair your ability to explore the image and think insightfully.

For all the provided images, in order, the i-th image has already been read into the global variable `image_clue_i` using the PIL.Image.open() function. When writing Python code, you can directly use these variables without needing to read them again.

You MUST use the python tool (e.g., matplotlib library) to inspect the image before authoring the question — zoom into at least one region that looks interesting (small text, faces, objects in clutter, fine textures, partially occluded items, etc.) and base the question on something you observed in the close-up views. This includes but is not limited to zooming in, rotating, adjusting contrast, computing statistics, or isolating features.

Note that when you use matplotlib to visualize data or further process images, you need to use plt.show() to display these images; there is no need to save them. Do not use image processing libraries like cv2 or PIL. If you want to check the value of a variable, you MUST use print() to check it.

The output (wrapped in `<interpreter>output_str</interpreter>`) can be returned to aid your inspection. The Python code should be complete scripts, including necessary imports.
Each code snippet is wrapped with `<code>
```python
code snippet
```
</code>`.

*image resolution:*

Image Width: {width}; Image Height: {height}

*your task:*
Explore the image with code first; do NOT propose the final MCQ until you have inspected at least one cropped region with the tool. When — and only when — you are ready to finalize, stop writing code and output a single fenced ```json ... ``` block with keys "question" (question text only, no options), "options" (A-D), "answer" (one letter):

```json
{
  "question": "<question text only, no options>",
  "options": {"A": "...", "B": "...", "C": "...", "D": "..."},
  "answer": "<one of A, B, C, D>"
}
```

The JSON block is the signal that you are done; do not emit it earlier.
\end{lstlisting}
\end{tcolorbox}
\section{Analysis}

\subsection{Per-Round Data Distribution}

We sample 100 images from the question pool used in each training round and 100 unique images from each evaluation benchmark, deduplicating benchmarks that attach multiple questions to one image. Every image is encoded independently by the frozen Qwen3-VL-4B vision tower, and we take the mean of its merged patch embeddings followed by L2 normalization as the image representation. All 1000 embeddings are then projected jointly: PCA first reduces the 2560-dimensional vectors to 50 dimensions, and t-SNE (perplexity 30, PCA initialization) maps them to two dimensions. Because a single shared projection is fit on all points, coordinates are directly comparable across the three per-round panels, which differ only in which round's generated images are highlighted.

\begin{figure}[h]
\centering
\includegraphics[width=1\linewidth]{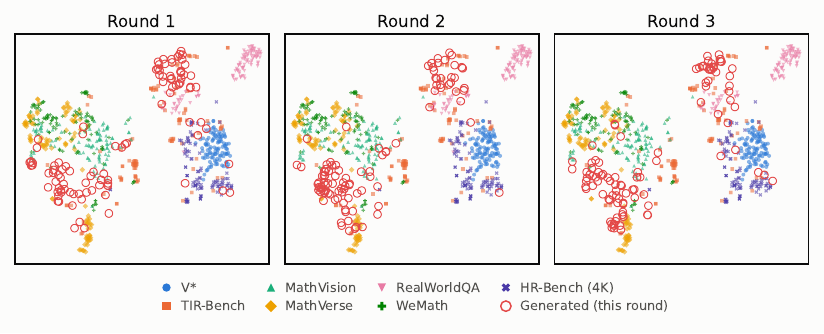}
\caption{\textbf{Image-embedding space of generated questions vs.\ evaluation benchmarks across training rounds.} Each panel shows a shared 2D projection (PCA + t-SNE over Qwen3-VL-4B vision-tower embeddings) of 100 images sampled from each of seven evaluation benchmarks (colored points) and 100 images underlying the questions generated in the indicated training round (red rings). 
% Generated images concentrate in the math-diagram region (MathVerse, WeMath, MathVision) and near TIR-Bench, while the natural-image benchmarks (V*, HR-Bench, RealWorldQA) form distinct clusters that the generated distribution largely does not cover, indicating that evaluation gains on these benchmarks arise from transfer rather than distributional overlap. The generated distribution remains broadly stable across rounds, with a mild shift toward diagram-style content in rounds 2–3.
}
\label{fig:per-round-dist}
\end{figure}

In each round, sampled images concentrate in the math-diagram region and near TIR-Bench, while the natural-image benchmarks form distinct clusters that the generated distribution largely does not cover, indicating that evaluation gains on these benchmarks arise from transfer rather than distributional overlap. The generated distribution remains broadly stable across rounds, with a mild shift toward diagram-style content in rounds 2–3. 
As these images constitute the sole input information during \methodname{}, their distribution naturally correlates with the per-round improvement. For example, the RealworldQA cluster is isolated from the others and is not covered by the training images, resulting in the smallest performance gains among all benchmarks.

\subsection{Iteration Scaling}

We continue to scale the evolution rounds of \methodname{} to 8 rounds. As reported in Tab.\ref{tab:iteration-scaling-bounds} instead of consistent improvement, \methodname{} evolution suffers with a degradation trend across evaluated tasks, similar to the scaling bound observed in R-Zero~\citep{huang2026r}. 

Concretely, most benchmarks reach their peak at round 5 or 6 and deteriorate thereafter.
MathVision-Mini drops steadily from 27.3 at round 5 to 23.4 at round 7, recovering only marginally to 23.7 at round 8.
V* peaks at 78.5 in round 6 before falling sharply to 72.3 in round 7, and V* and MathVision-Mini both end round 8 well below their round-5 scores. Text-centric benchmarks degrade more mildly: TIR-Bench declines from 21.7 to 21.0, while WeMath oscillates in the 62--64 range and never exceeds its round-5 score by a meaningful margin.

\begin{table}[h]
    \centering
    \begin{tabular}[t]{lllll}
    \toprule & Round 5 & Round 6 & Round 7 & Round 8 \\
    \midrule
    WeMath & 63.2 & 62.1 & 64.2 & 64.3 \\
    MathVision-Mini & 27.3 & 24.0 & 23.4 & 23.7 \\
    V* & 77.5 & 78.5 & 72.3 & 74.4 \\
    TIR-Bench & 21.7 & 20.7 & 21.1 & 21.0 \\
    \bottomrule
    \end{tabular}
    \caption{\textbf{Performance of \methodname{} across evolution rounds 5--8 on four benchmarks.} Scores are reported in percentages; the best result for each benchmark is achieved no later than round 6, after which performance degrades.}
    \label{tab:iteration-scaling-bounds}
\end{table}

We attribute the instability and degradation under iteration scaling to the quality degradation of the generated solver training data.
To verify this hypothesis, we use Qwen3.8-27B~\citep{qwen38} to probe the accuracy of the proposed questions $\{(I_i, x_i, y_i)\}_{i=1}^N$ at each round.
Note that even though the proposed labels are guaranteed by the consensus between the challenger and the solver, Table~\ref{tab:data_probe} shows that the proposed questions are in fact getting worse: judge accuracy starts at 68.5\% in round 1, remains stable around 62--67\% through round 4, and then drops sharply to 58.7\% at round 5 and below 47\% from round 6 onward, ending at 43.6\% in round 8.
This decline closely mirrors the benchmark trend in Table~\ref{tab:iteration-scaling-bounds}, where performance peaks around rounds 5--6 and deteriorates thereafter, suggesting that the scaling bound is driven by the decay of training signal quality rather than by optimization limits.
R-Zero relies on an additional round of majority voting from solver answers to generate the answer $y$, but as both models drift away from the true data distribution over iterations, the consensus between two degraded models is no longer a faithful filter, and it suffers from the same degradation.

\begin{table}[h]
    \centering
    \begin{tabular}{lcccccccc}
    \toprule
    Round & 1 & 2 & 3 & 4 & 5 & 6 & 7 & 8 \\
    \midrule
    $n$ & 1201 & 1280 & 1281 & 1280 & 1281 & 1280 & 1280 & 1280 \\
    Judge Acc. (\%) & 68.5 & 62.2 & 66.1 & 67.1 & 58.7 & 46.1 & 46.2 & 43.6 \\
    \bottomrule
    \end{tabular}
    \caption{Accuracy of the proposed solver training questions $\{(I_i, x_i, y_i)\}_{i=1}^N$ at each evolution round, probed by Qwen3.8-27B. Despite consensus-based filtering, question quality degrades from 68.5\% to 43.6\%, mirroring the performance degradation in Table~\ref{tab:iteration-scaling-bounds}.}
    \label{tab:data_probe}
\end{table}
\section{Case Study}

\subsection{Challenger Example Trace}

We present a full real trajectory in Fig.~\ref{fig:example-trace-challenger} about how a challenger equipped with PyVision harness proposes a visual question pair from an unlabeled image.

\begin{figure}[p]
    \centering
    \includegraphics[width=\linewidth]{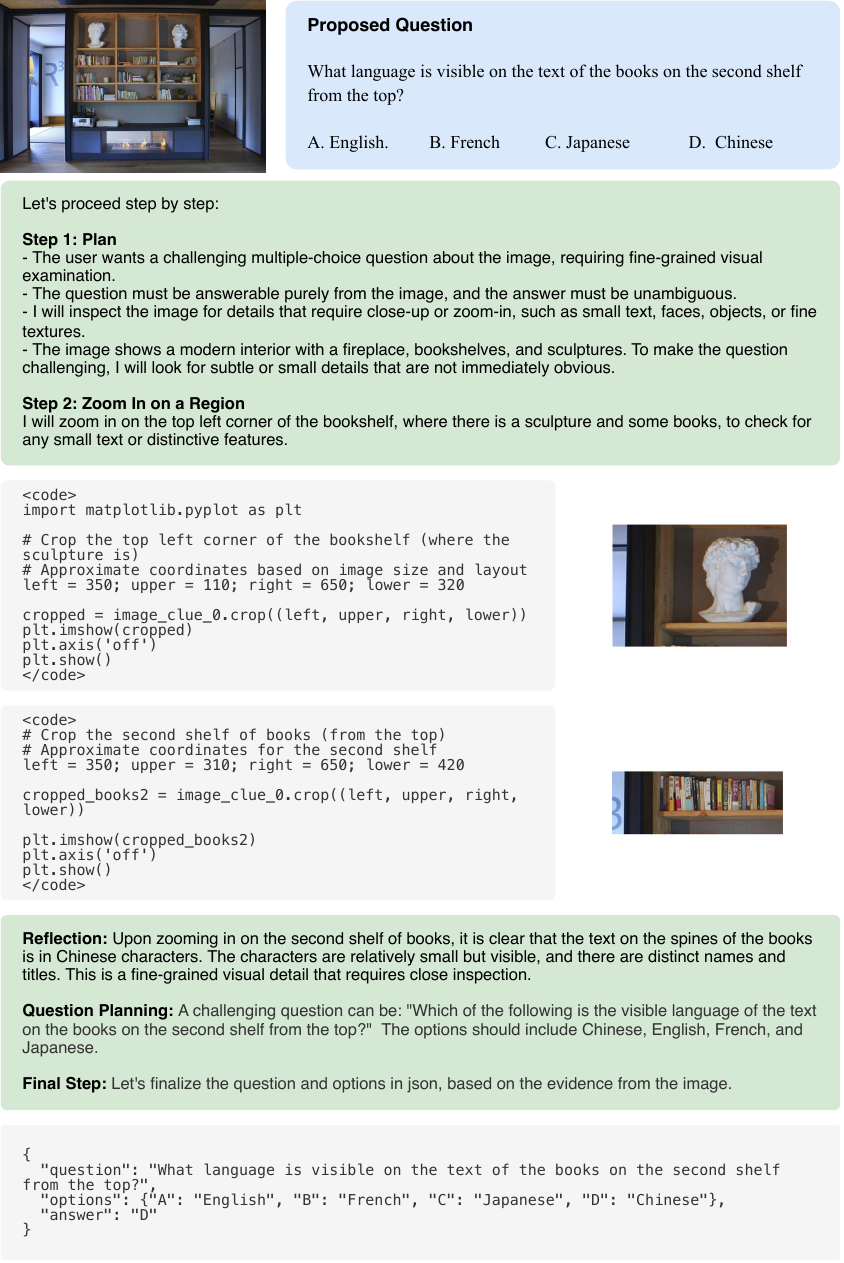}
    \caption{Example Trajectory of Challenger Agent.}
    \label{fig:example-trace-challenger}
\end{figure}

\subsection{Example Proposed Questions}

We present some of proposed questions on the same set of images in Tab.~\ref{tab:question-cases}. As \methodname{} proceeds, the complexity of proposed question increases.

\begin{table}[h]
\centering
\caption{Proposed Questions.}
\includegraphics[width=1\linewidth]{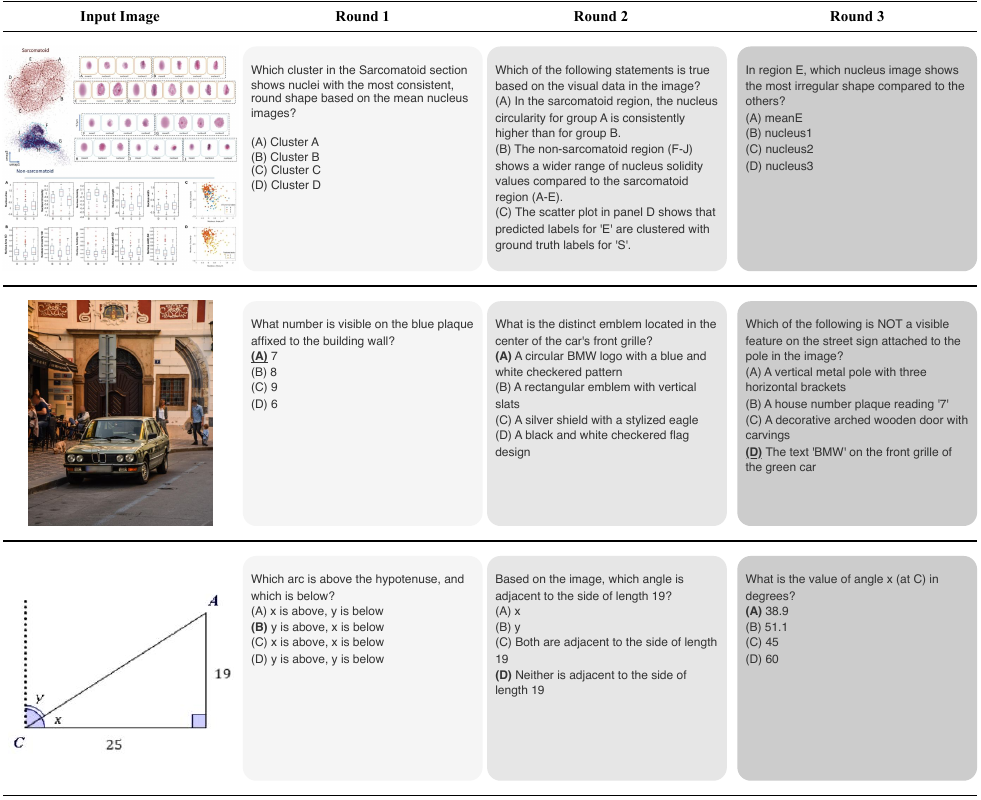}
\label{tab:question-cases}
\end{table}

\end{document}